\pdfoutput=1
\documentclass[letterpaper]{article}
\usepackage[totalwidth=480pt, totalheight=680pt]{geometry}

\usepackage[utf8]{inputenc} 
\usepackage[T1]{fontenc}    
\usepackage[dvipsnames]{xcolor}         
\usepackage[pagebackref,colorlinks=true,urlcolor=blue,linkcolor=blue,citecolor=OliveGreen]{hyperref}               
\usepackage{url}            
\usepackage{booktabs}       
\usepackage{amsfonts}       
\usepackage{nicefrac}       
\usepackage{microtype}      
\usepackage[T1]{fontenc}
\usepackage{lmodern}

\usepackage{algorithm}
\usepackage{algorithmicx}
\usepackage{algpseudocode}
\usepackage{subcaption}

\usepackage{amsmath}
\usepackage{amssymb}
\usepackage{amsthm}
\usepackage{pifont}
\usepackage{xcolor}
\definecolor{darkgreen}{rgb}{0,0.5,0}
\usepackage{hyperref}
\hypersetup{
    unicode=false,          
    colorlinks=true,        
    linkcolor=red,          
    citecolor=darkgreen,    
    filecolor=magenta,      
    urlcolor=blue           
}
\usepackage[capitalize, nameinlink]{cleveref}

\usepackage{booktabs}

\usepackage{thm-restate}
\theoremstyle{plain}
\newtheorem{theorem}{Theorem}

\newtheorem{lemma}[theorem]{Lemma}

\theoremstyle{definition}
\newtheorem{definition}[theorem]{Definition}

\theoremstyle{remark}

\newtheorem{conjecture}[theorem]{Conjecture}

\newtheorem*{remark*}{Remark}

\usepackage{tikz}
\usetikzlibrary{calc, graphs, graphs.standard, shapes, arrows, positioning, decorations.pathreplacing, decorations.markings, decorations.pathmorphing, fit, matrix, patterns, shapes.misc, tikzmark}

\newcommand{\R}{\mathbb{R}}
\newcommand{\E}{\mathbb{E}}
\newcommand{\cI}{\mathcal{I}}

\newcommand{\cE}{\mathcal{E}}
\newcommand{\cO}{\mathcal{O}}
\newcommand{\argmax}{\text{argmax}}
\newcommand{\argmin}{\text{argmin}}
\newcommand{\cmark}{\ding{51}}
\newcommand{\xmark}{\ding{55}}

\title{Verification and search algorithms for causal DAGs}
\author{
Davin Choo\thanks{Equal contribution}\\
National University of Singapore\\
\texttt{davin@u.nus.edu}
\and
Kirankumar Shiragur\footnotemark[1]\\
Stanford University\\
\texttt{shiragur@stanford.edu}
\and
Arnab Bhattacharyya\\
National University of Singapore\\
\texttt{arnabb@nus.edu.sg}
}
\date{}

\begin{document}

\maketitle

\begin{abstract}
We study two problems related to recovering causal graphs from interventional data:
(i) \emph{verification}, where the task is to check if a purported causal graph is correct, and (ii) \emph{search}, where the task is to recover the correct causal graph.
For both, we wish to minimize the number of interventions performed.
For the first problem, we give a characterization of a minimal sized set of atomic interventions that is necessary and sufficient to check the correctness of a claimed causal graph.
Our characterization uses the notion of \emph{covered edges}, which enables us to obtain simple proofs and also easily reason about earlier known results.
We also generalize our results to the settings of bounded size interventions and node-dependent interventional costs.
For all the above settings, we provide the first known provable algorithms for efficiently computing (near)-optimal verifying sets on general graphs.
For the second problem, we give a simple adaptive algorithm based on graph separators that produces an atomic intervention set which fully orients any essential graph while using $\cO(\log n)$ times the optimal number of interventions needed to \emph{verify} (verifying size) the underlying DAG on $n$ vertices.
This approximation is tight as \emph{any} search algorithm on an essential line graph has worst case approximation ratio of $\Omega(\log n)$ with respect to the verifying size.
With bounded size interventions, each of size $\leq k$, our algorithm gives an $\cO(\log n \cdot \log k)$ factor approximation.
Our result is the first known algorithm that gives a non-trivial approximation guarantee to the verifying size on general unweighted graphs and with bounded size interventions.

\end{abstract}

\section{Introduction}
\label{sec:introduction}

Causal inference has long been an important concept in various fields such as philosophy \cite{reichenbach1956direction,woodward2005making,eberhardt2007interventions}, medicine/biology/genetics \cite{king2004functional,sverchkov2017review,rotmensch2017learning,pingault2018using}, and econometrics \cite{hoover1990logic,rubin2006estimating}.
Recently, there has also been a growing interest in the machine learning community to use causal inference techniques to improve generalizability to novel testing environments (e.g.\ see \cite{ganin2016domain,louppe2017learning,arjovsky2019invariant,scholkopf2022causality} and references therein).
Under the assumption of causal sufficiency, where there is no unobserved confounders or selection bias, causal inference using observational data has been extensively studied and many algorithms such as PC \cite{spirtes2000causation} and GES \cite{chickering2002optimal} have been proposed.
These algorithms typically recover a causal graph up to its Markov equivalence class and it is known to be fundamentally impossible to learn causal relationships solely based on observational data.
To overcome this issue, one either adds data modeling assumptions (e.g.\ see \cite{shimizu2006linear,peters2014identifiability,mooij2016distinguishing}) or performs interventions to obtain interventional data (see \cref{sec:related-work} for a literature review).
In our work, we study the causal discovery problem via interventions.
As interventions often correspond to real-world experimental trials, they can be costly and it is of practical importance to minimize the number or cost of interventions.

In this work, we consider \emph{ideal interventions} (i.e.\ hard interventions with infinite samples\footnote{We do not consider sample complexity issues in this work.}) to recover causal graphs from its Markov equivalence class -- the \emph{Markov equivalence class} (MEC) of $G$, denoted by $[G]$, is the set of graphs that encode the same conditional distributions and it is known that any MEC $[G]$ can be represented by a unique partially oriented \emph{essential graph} $\cE(G)$.
While ideal interventions may not always be possible in practice, they serve as a first step to understand the verification and search problems.
Such interventions give us a clean graph-theoretic way to reason and identify arc directions in the causal graph $G = (V,E)$.
With these interventions, it is known that intervening on a set $S \subseteq V$ allows us to infer the edge orientation of any edge cut by $S$ and $V \setminus S$ \cite{eberhardt2007causation,hyttinen2013experiment,hu2014randomized,shanmugam2015learning,kocaoglu2017cost}.

Using the ideal interventions, we solve the verification and search problems.
The \emph{search problem} is the traditional question of finding a minimum set of interventions to recover the underlying ground truth causal graph.
As for the \emph{verification problem}, consider the following scenario:
Suppose we have performed a observational study to obtain a MEC and consulted an expert about the identity of the ground truth causal graph.
The question of verification involves testing if the expert is correct using the minimal number of interventional studies.

\begin{definition}[Search problem]
\label{defn:search-problem}
Given the essential graph $\cE(G^*)$ of an unknown causal graph $G^*$, use the minimal number of interventions to fully recover the ground truth causal graph $G^*$.
\end{definition}

\begin{definition}[Verification problem]
\label{defn:verification-problem}
Given the essential graph $\cE(G^*)$ of an unknown causal graph $G^*$ and an expert's graph $G \in [G^*]$, use the minimal number of interventions to verify $G \stackrel{?}{=} G^*$.
\end{definition}

To solve both these problems, we compute \emph{verifying sets} (see \cref{def:min-verifying-set}).
A verifying set is a collection of interventions that fully orients the essential graph.
Note that the minimum size/cost verifying set serves as a natural lower bound for both the search and verification problems.
One of our key contributions is to efficiently compute these verifying sets for a given causal graph.

\paragraph{Contributions}
We study the problems of verification and search for a causal graph from its MEC using ideal interventions.
In our work, we make standard assumptions such as the Markov assumption, the faithfulness assumption, and causal sufficiency \cite{spirtes2000causation}.

\begin{enumerate}
    \item \textbf{Verification}:
    We provide the first known efficient algorithms for computing minimal sized atomic verifying sets and near-optimal bounded size verifying sets (that use at most one more intervention than optimal) on general graphs.
    When vertices have additive interventional costs according to a weight function $w: V \to \R$, we give efficient computation of verifying sets $\cI$ that minimizes $\alpha \cdot w(\cI) + \beta \cdot |\cI|$.
    Our atomic verifying sets have optimal cost and bounded size verifying sets incur a total additive cost of at most $2 \beta$ more than optimal.
    To achieve these results, we prove properties about covered edges and give a characterization of verifying sets as a separation of unoriented covered edges in the given essential graph.
    Using our covered edge perspective, we show that the universal lower bounds of \cite{squires2020active,porwal2021almost} are \emph{not} tight and give a simple proof recovering the verification upper bound of \cite{porwal2021almost}.
    \item \textbf{Adaptive search}:
    We consider adaptive search algorithms which produce a \emph{sequence} of interventions one-at-a-time, possibly using information gained from the outcomes of earlier chosen interventions.
    Building upon the lower bound of \cite{squires2020active}, we establish a stronger (but not computable) lower bound to verify a causal graph.
    This further implies a lower bound on the minimum interventions needed by \emph{any} adaptive search algorithm.
    We also provide an adaptive search algorithm (based on graph separators for chordal graphs) and use our lower bound to prove that our approach uses at most a logarithmic multiplicative factor more interventions than a minimum sized verifying set of the true underlying causal graph.
\end{enumerate}

\textbf{Outline}
\cref{sec:prelim} introduces notation and preliminary notions, as well as discuss some known results on ideal interventions.
We give our results in \cref{sec:results} and provide an overview of the techniques used in \cref{sec:techniques}.
We discuss our experimental results in \cref{sec:experiments}.
Full proofs, side discussions, and source code/scripts are given in the appendix.

\section{Preliminaries and related work}
\label{sec:prelim}

For any set $A$, we denote its powerset by $2^A$.
We write $\{1, \ldots, n\}$ as $[n]$ and hide absolute constant multiplicative factors in $n$ using asymptotic notations $\cO(\cdot)$, $\Omega(\cdot)$, and $\Theta(\cdot)$.
Throughout, we use $G^*$ to denote the (unknown) ground truth DAG and we only know its essential graph $\cE(G^*)$.

\subsubsection*{Graph notions}
We study partially directed graphs $G = (V, E, A)$ on $|V| = n$ vertices with unoriented edges $E$ and oriented arcs $A$.
Any possible edge between two distinct vertices $u,v \in V$ is either undirected, oriented in one direction, or absent in the graph.
We write $u \sim v$ if these vertices are connected (either through an unoriented edge or an arc) in the graph and $u \not\sim v$ to indicate the absence of any edge/arc connection.
If $(u,v) \in A$ or $(v,u) \in A$, we write $u \to v$ or $u \gets v$ respectively.
When $E = \emptyset$, we say that the graph is fully oriented.
For any graph $G$, we use $V(G), E(G), A(G)$ to denote its vertices, unoriented edges and oriented arcs.

For any subset of vertices $V'\subseteq V$, $G[V']$ denotes the vertex-induced subgraph on $V'$.
Similarly, we define $G[A']$ and $G[E']$ as arc-induced and edge-induced subgraphs of $G$ for $A' \subseteq A$ and $E' \subseteq E$ respectively.
For an undirected graph $G$, $\omega(G)$ refers to the size of its maximum clique and $\chi(G)$ refers to its chromatic number.
For any vertex $v \in V$ in a directed graph, $Pa(v) \subseteq V$ denotes the parent set of $v$ and $pa(v)$ denotes the vector of values taken by $v$'s parents.

The \emph{skeleton} of a graph $G$ refers to the graph $G' = (V, E \cup A, \emptyset)$ where all arcs are made unoriented.
A \emph{v-structure} refers to three distinct vertices $u,v,w \in V$ such that $u \to v \gets w$ and $u \not\sim w$.
A simple cycle is a sequence of $k \geq 3$ vertices where $v_1 \sim v_2 \sim \ldots \sim v_k \sim v_1$.
The cycle is directed if at least one of the edges is directed and all directed arcs are in the same direction along the cycle.
A partially directed graph is a \emph{chain graph} if it contains no directed cycle.
In the undirected graph $G'$ obtained by removing all arcs from a chain graph $G$, each connected component in $G'$ is called a \emph{chain component}.
We use $CC(G)$ to denote the set of chain components in $G$.
Note that the vertices of these chain components form a partition of $V$.

Directed acyclic graphs (DAGs), a special case of chain graphs where \emph{all} edges are directed, are commonly used as graphical causal models \cite{pearl2009causality} where vertices represents random variables and the joint probability density $f$ factorizes according to the Markov property:
$
f(v_1, \ldots, v_n) = \prod_{i=1}^n f(v_i \mid pa(v))
$.
We can associate a \emph{valid permutation / topological ordering} $\pi : V \to [n]$ to any (partially oriented) DAG such that oriented arcs $(u,v)$ satisfy $\pi(u) < \pi(v)$ and any unoriented arc $\{u,v\}$ can be oriented as $u \to v$ whenever $\pi(u) < \pi(v)$ without forming directed cycles.
While there may be multiple valid permutations, we often only care that there exists at least one such permutation.
For any DAG $G$, we denote its \emph{Markov equivalence class} (MEC) by $[G]$ and \emph{essential graph} by $\cE(G)$.
It is known that two graphs are Markov equivalent if and only if they have the same skeleton and v-structures \cite{verma1990,andersson1997characterization}.

A \emph{clique} is a graph where $u \sim v$ for any pair of vertices $u,v \in V$.
A \emph{maximal clique} is an vertex-induced subgraph of a graph that is a clique and ceases to be one if we add any other vertex to the subgraph.
If all edges in the clique are oriented in an acyclic manner, then there is a unique valid permutation $\pi$ that respects this orientation.
We denote $v = \argmax_{u \in V} \pi(u)$ as the \emph{sink} of the clique.

\subsubsection*{Interventions, verifying sets, and additive vertex intervention costs}
An \emph{intervention} $S \subseteq V$ is an experiment where the experimenter forcefully sets each variable $s \in S$ to some value, independent of the underlying causal structure.
An intervention is called an \emph{atomic intervention} if $|S| = 1$ and called a \emph{bounded size intervention} if $|S| \leq k$ for some size upper bound $k$.
One can view observational data as a special case where $S = \emptyset$.
Interventions affect the joint distribution of the variables and are formally captured by Pearl's do-calulus \cite{pearl2009causality}.
An \emph{intervention set} $\cI \subseteq 2^V$ is a collection of interventions and $\cup_{S \in \cI} S$ is the union of all intervened vertices.

In this work, we study ideal interventions.
Graphically speaking, an ideal intervention $S$ on $G$ induces an interventional graph $G_S$ where all incoming arcs to vertices $v \in S$ are removed \cite{eberhardt2012number} and it is known that intervening on a set $S \subseteq V$ allows us to infer the edge orientation of any edge cut by $S$ and $V \setminus S$ \cite{eberhardt2007causation,hyttinen2013experiment,hu2014randomized,shanmugam2015learning,kocaoglu2017cost}.
For ideal interventions, an $\cI$-essential graph $\cE_{\cI}(G)$ of $G$ is the essential graph representing the Markov equivalence class of graphs whose interventional graphs for each intervention is Markov equivalent to $G_S$ for any intervention $S \in \cI$.

There are several known properties about $\cI$-essential graph properties \cite{hauser2012characterization,hauser2014two,squires2020active}:
Every $\cI$-essential graph is a chain graph with chordal\footnote{A chordal graph is a graph where every cycle of length at least 4 has a chord, which is an edge that is not part of the cycle but connects two vertices of the cycle. See \cite{blair1993introduction} for more properties.} chain components.
This includes the case of $S = \emptyset$.
Orientations in one chain component do not affect orientations in other components.
In other words, to fully orient any essential graph $\cE(G^*)$, it is necessary and sufficient to orient every chain component in $\cE(G^*)$ independently.
More formally, we have\footnote{Lemma 1 of \cite{hauser2014two} actually considers a \emph{single} additional intervention, but a closer look at their proof shows that the statement can be strengthened to allow for \emph{multiple} additional interventions. For completeness, we provide the proof of this strengthened version in \cref{sec:appendix-stronger-HB}. Note that we can drop the $\emptyset$ intervention in the statement since essential graphs are defined with the observational data provided.}

\begin{restatable}[Modified lemma 1 of \cite{hauser2014two}]{lemma}{hauserbulmannstrengthened}
\label{lem:hauser-bulmann-strengthened}
Let $\cI \subseteq 2^V$ be an intervention set.
Consider the $\cI$-essential graph $\cE_{\cI}(G^*)$ of some DAG $G^*$ and let $H \in CC(\cE_{\cI}(G^*))$ be one of its chain components.
Then, for any additional interventional set $\cI' \subseteq 2^V$ such that $\cI \cap \cI' = \emptyset$, we have
\[
\cE_{\cI \cup \cI'}(G^*)[V(H)] = \cE_{\{S \cap V(H)~:~S \in \cI'\}}(G^*[V(H)]).
\]
\end{restatable}

As a consequence of \cref{lem:hauser-bulmann-strengthened}, one may assume without loss of generality that $CC(\cE(G))$ is a single connected component and then generalize results by summing across all connected components.

A \emph{verifying set} $\cI$ for a DAG $G \in [G^*]$ is an intervention set that fully orients $G$ from $\cE(G^*)$, possibly with repeated applications of Meek rules (see \cref{sec:appendix-meek-rules}).
In other words, for any graph $G = (V,E)$ and any verifying set $\cI$ of $G$, we have $\cE_{\cI}(G)[V'] = G[V']$ for \emph{any} subset of vertices $V' \subseteq V$.
Furthermore, if $\cI$ is a verifying set for $G$, then $\cI \cup S$ is also a verifying set for $G$ for any additional intervention $S \subseteq V$.
While DAGs may have multiple verifying sets in general, we are often interested in finding one with minimum size or cost.

\begin{definition}[Minimum size/cost verifying set]
\label{def:min-verifying-set}
Let $w$ be a weight function on intervention sets.
An intervention set $\cI$ is called a verifying set for a DAG $G^*$ if $\cE_{\cI}(G^*) = G^*$.
$\cI$ is a \emph{minimum size (resp.\ cost) verifying set} if $\cE_{\cI'}(G^*) \neq G^*$ for any $|\cI'| < |\cI|$ (resp.\ for any $w(\cI') < w(\cI)$).
\end{definition}

When restricting to interventions of size at most $k$, the \emph{minimum verification number} $\nu_k(G)$ of $G$ denotes the size of the minimum size verifying set for any DAG $G \in [G^*]$.
That is, any revealed arc directions when performing interventions on $\cE(G^*)$ respects $G$.
$\nu_1(G)$ denotes the case where we restrict to atomic interventions.
One of the goals of this work is to characterize $\nu(G)$ given an essential graph $\cE(G^*)$ and some $G \in [G^*]$.

\subsubsection*{Covered edges}
Covered edges are special arcs in a causal graph where the endpoints of $u \sim v$ share the same set of parents in $V \setminus \{u,v\}$.
These edges are crucial in causal discovery because their orientation can be reversed and they still yield the same conditional independencies.
See \cref{fig:standing-windmill} for an illustration.
Note that one can compute all covered edges of a given DAG $G$ in polynomial time.

\begin{figure}[htbp]
\centering
\resizebox{\linewidth}{!}{%
\begin{tikzpicture}
%
%
\node[draw, circle, minimum size=10pt, inner sep=0pt] at (0,0) (a-eg) {\scriptsize $a$};
\node[draw, circle, minimum size=10pt, inner sep=0pt] at ($(a-eg) + (-1,-0.5)$) (b-eg) {\scriptsize $b$};
\node[draw, circle, minimum size=10pt, inner sep=0pt] at ($(a-eg) + (-1,0.5)$) (c-eg) {\scriptsize $c$};
\node[draw, circle, minimum size=10pt, inner sep=0pt] at ($(a-eg) + (-0.5,1)$) (d-eg) {\scriptsize $d$};
\node[draw, circle, minimum size=10pt, inner sep=0pt] at ($(a-eg) + (0.5,1)$) (e-eg) {\scriptsize $e$};
\node[draw, circle, minimum size=10pt, inner sep=0pt] at ($(a-eg) + (1,0.5)$) (f-eg) {\scriptsize $f$};
\node[draw, circle, minimum size=10pt, inner sep=0pt] at ($(a-eg) + (1,-0.5)$) (g-eg) {\scriptsize $g$};
\node[draw, circle, minimum size=10pt, inner sep=0pt] at ($(a-eg) + (0,-1)$) (h-eg) {\scriptsize $h$};

\draw[thick] (a-eg) -- (b-eg);
\draw[thick] (a-eg) -- (c-eg);
\draw[thick] (a-eg) -- (d-eg);
\draw[thick] (a-eg) -- (e-eg);
\draw[thick] (a-eg) -- (f-eg);
\draw[thick] (a-eg) -- (g-eg);
\draw[thick] (a-eg) -- (h-eg);
\draw[thick] (b-eg) -- (c-eg);
\draw[thick] (d-eg) -- (e-eg);
\draw[thick] (f-eg) -- (g-eg);

%
%
\node[draw, circle, minimum size=10pt, inner sep=0pt] at (4,0) (a-gstar) {\scriptsize $a$};
\node[draw, circle, minimum size=10pt, inner sep=0pt] at ($(a-gstar) + (-1,-0.5)$) (b-gstar) {\scriptsize $b$};
\node[draw, circle, minimum size=10pt, inner sep=0pt] at ($(a-gstar) + (-1,0.5)$) (c-gstar) {\scriptsize $c$};
\node[draw, circle, minimum size=10pt, inner sep=0pt] at ($(a-gstar) + (-0.5,1)$) (d-gstar) {\scriptsize $d$};
\node[draw, circle, minimum size=10pt, inner sep=0pt] at ($(a-gstar) + (0.5,1)$) (e-gstar) {\scriptsize $e$};
\node[draw, circle, minimum size=10pt, inner sep=0pt] at ($(a-gstar) + (1,0.5)$) (f-gstar) {\scriptsize $f$};
\node[draw, circle, minimum size=10pt, inner sep=0pt] at ($(a-gstar) + (1,-0.5)$) (g-gstar) {\scriptsize $g$};
\node[draw, circle, minimum size=10pt, inner sep=0pt] at ($(a-gstar) + (0,-1)$) (h-gstar) {\scriptsize $h$};

\draw[thick, -stealth] (a-gstar) -- (b-gstar);
\draw[thick, -stealth] (a-gstar) -- (c-gstar);
\draw[thick, -stealth] (a-gstar) -- (d-gstar);
\draw[thick, -stealth] (a-gstar) -- (e-gstar);
\draw[thick, -stealth] (a-gstar) -- (f-gstar);
\draw[thick, -stealth] (a-gstar) -- (g-gstar);
\draw[thick, -stealth, dashed] (h-gstar) -- (a-gstar);
\draw[thick, -stealth, dashed] (b-gstar) -- (c-gstar);
\draw[thick, -stealth, dashed] (d-gstar) -- (e-gstar);
\draw[thick, -stealth, dashed] (f-gstar) -- (g-gstar);

%
%
\node[draw, circle, minimum size=10pt, inner sep=0pt] at (8,0) (a-g1) {\scriptsize $a$};
\node[draw, circle, minimum size=10pt, inner sep=0pt] at ($(a-g1) + (-1,-0.5)$) (b-g1) {\scriptsize $b$};
\node[draw, circle, minimum size=10pt, inner sep=0pt] at ($(a-g1) + (-1,0.5)$) (c-g1) {\scriptsize $c$};
\node[draw, circle, minimum size=10pt, inner sep=0pt] at ($(a-g1) + (-0.5,1)$) (d-g1) {\scriptsize $d$};
\node[draw, circle, minimum size=10pt, inner sep=0pt] at ($(a-g1) + (0.5,1)$) (e-g1) {\scriptsize $e$};
\node[draw, circle, minimum size=10pt, inner sep=0pt] at ($(a-g1) + (1,0.5)$) (f-g1) {\scriptsize $f$};
\node[draw, circle, minimum size=10pt, inner sep=0pt] at ($(a-g1) + (1,-0.5)$) (g-g1) {\scriptsize $g$};
\node[draw, circle, minimum size=10pt, inner sep=0pt] at ($(a-g1) + (0,-1)$) (h-g1) {\scriptsize $h$};

\draw[thick, -stealth, dashed] (a-g1) -- (b-g1);
\draw[thick, -stealth] (a-g1) -- (c-g1);
\draw[thick, -stealth, dashed] (a-g1) -- (d-g1);
\draw[thick, -stealth] (a-g1) -- (e-g1);
\draw[thick, -stealth, dashed] (a-g1) -- (f-g1);
\draw[thick, -stealth] (a-g1) -- (g-g1);
\draw[thick, -stealth, dashed] (a-g1) -- (h-g1);
\draw[thick, -stealth, dashed] (b-g1) -- (c-g1);
\draw[thick, -stealth, dashed] (d-g1) -- (e-g1);
\draw[thick, -stealth, dashed] (f-g1) -- (g-g1);

%
%
\node[draw, circle, minimum size=10pt, inner sep=0pt] at (12,0) (a-g2) {\scriptsize $a$};
\node[draw, circle, minimum size=10pt, inner sep=0pt] at ($(a-g2) + (-1,-0.5)$) (b-g2) {\scriptsize $b$};
\node[draw, circle, minimum size=10pt, inner sep=0pt] at ($(a-g2) + (-1,0.5)$) (c-g2) {\scriptsize $c$};
\node[draw, circle, minimum size=10pt, inner sep=0pt] at ($(a-g2) + (-0.5,1)$) (d-g2) {\scriptsize $d$};
\node[draw, circle, minimum size=10pt, inner sep=0pt] at ($(a-g2) + (0.5,1)$) (e-g2) {\scriptsize $e$};
\node[draw, circle, minimum size=10pt, inner sep=0pt] at ($(a-g2) + (1,0.5)$) (f-g2) {\scriptsize $f$};
\node[draw, circle, minimum size=10pt, inner sep=0pt] at ($(a-g2) + (1,-0.5)$) (g-g2) {\scriptsize $g$};
\node[draw, circle, minimum size=10pt, inner sep=0pt] at ($(a-g2) + (0,-1)$) (h-g2) {\scriptsize $h$};

\draw[thick, -stealth, dashed] (b-g2) -- (a-g2);
\draw[thick, -stealth, dashed] (a-g2) -- (c-g2);
\draw[thick, -stealth] (a-g2) -- (d-g2);
\draw[thick, -stealth] (a-g2) -- (e-g2);
\draw[thick, -stealth] (a-g2) -- (f-g2);
\draw[thick, -stealth] (a-g2) -- (g-g2);
\draw[thick, -stealth] (a-g2) -- (h-g2);
\draw[thick, -stealth] (b-g2) -- (c-g2);
\draw[thick, -stealth, dashed] (d-g2) -- (e-g2);
\draw[thick, -stealth, dashed] (f-g2) -- (g-g2);

%
%
\node[above=35pt of a-eg, inner sep=0pt] {\small $\cE(G^*)$};
\node[above=35pt of a-gstar, inner sep=0pt] {\small $G^*$};
\node[above=35pt of a-g1, inner sep=0pt] {\small $G_1$};
\node[above=35pt of a-g2, inner sep=0pt] {\small $G_2$};

%
%
\node[double arrow, draw, minimum height=2.5em, double arrow head extend=0.5ex, inner sep=2pt] at (6,0) (gstar-g1-arrow) {};
\node[above=5pt of gstar-g1-arrow] {\footnotesize $a \sim h$};

\node[double arrow, draw, minimum height=2.5em, double arrow head extend=0.5ex, inner sep=2pt] at (10,0) (g1-g2-arrow) {};
\node[above=5pt of g1-g2-arrow] {\footnotesize $a \sim b$};
\end{tikzpicture}
}
\caption{
A DAG $G^*$ with its essential graph $\cE(G^*)$ on the left.
$G_1$ and $G_2$ are two other DAGs that belong to the same Markov equivalence class $[G^*]$.
Dashed arcs are covered edges in each DAG.
One can perform a sequence of covered edge reversals to transform between the DAGs (see \cref{lem:sequence}).
Note that the sizes of the minimum vertex cover of the covered edges may differ across DAGs.
}
\label{fig:standing-windmill}
\end{figure}
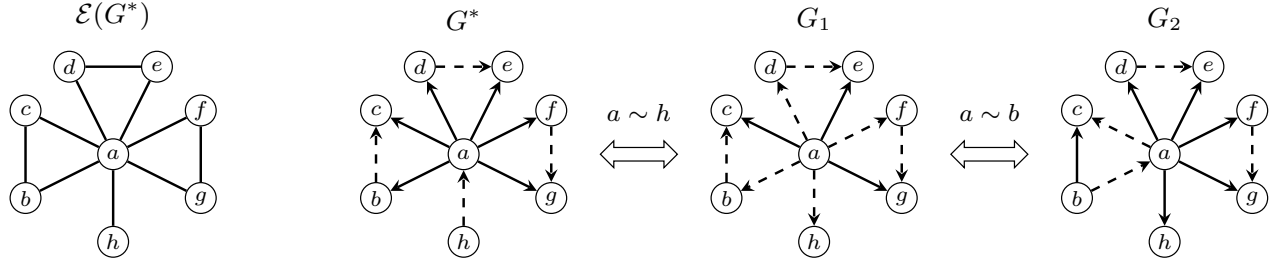

\begin{definition}[Covered edge]
An edge $u \sim v$ is a covered edge if $Pa(u) \setminus \{v\} = Pa(v) \setminus \{u\}$.
\end{definition}

\begin{definition}[Covered edge reversal]
A covered edge reversal means that we replace $u \to v$ with $v \to u$, for some covered edge $u \to v$, while keeping all other arcs unchanged.
\end{definition}

\begin{lemma}[\cite{chickering2013transformational}]
\label{lem:sequence}
If $G$ and $G'$ belong in the same MEC if and only if there exists a sequence of covered edge reversals to transform between them.
\end{lemma}

\begin{definition}[Separation of covered edges]
We say that an intervention $S \subseteq V$ \emph{separates} a covered edge $u \sim v$ if $|\{u,v\} \cap S| = 1$.
That is, \emph{exactly} one of the endpoints is intervened by $S$.
We say that an intervention set $\cI$ separates a covered edge $u \sim v$ if there exists $S \in \cI$ that separates $u \sim v$.
\end{definition}

\subsubsection*{Relationship between searching and verification}
Recall \cref{defn:search-problem} and \cref{defn:verification-problem}.
The verification number is a useful analytical tool for the search problem as $\nu_k(G^*)$ is a lower bound on the number of interventions used by an optimal search algorithm.
Furthermore, since the search problem needs to fully orient $\cE(G^*)$ regardless of which DAG is the ground truth, any search algorithm given $\cE(G^*)$ requires \emph{at least} $\min_{G \in [G^*]} \nu_k(G)$ interventions, even if it is adaptive and randomized.
In fact, the \emph{strongest possible universal lower bound} guarantee one can prove must be \emph{at most} $\min_{G \in [G^*]} \nu_k(G)$ and the \emph{strongest possible universal upper bound} guarantee one can prove must be \emph{at least} $\max_{G \in [G^*]} \nu_k(G)$.
Note that if the search algorithm is \emph{non-adaptive}, then it trivially needs \emph{at least} $\max_{G \in [G^*]} \nu_k(G)$ interventions.

\textbf{Example}
Consider the graph $G^*$ in \cref{fig:standing-windmill} where the essential graph $\cE(G^*)$ representing the MEC $[G^*]$ is the standing windmill\footnote{To be precise, it is the Wd(3,3) windmill graph with an additional edge from the center.}.
Now, consider only atomic interventions.
We will later show that the minimum verification number of a DAG is the size of the minimum vertex cover of its covered edges (see \cref{thm:efficient-optimal-atomic}).
One can check that $\nu_1(G^*) = \nu_1(G_1) = 4$ while $\nu_1(G_2) = 3$.
In fact, we actually show that $\min_{G \in [G^*]} \nu_1(G) = 3$ and $\max_{G \in [G^*]} \nu_1(G) = 4$ in \cref{sec:appendix-standing-windmill}.
Thus, any search algorithm using only atomic interventions on $\cE(G^*)$ needs at least 3 atomic interventions.

\subsection{Related work}
\label{sec:related-work}

\cref{tab:upper-bounds} and \cref{tab:lower-bounds} in \cref{sec:appendix-landscape} summarize\footnote{Some known results are discussed in further detail below instead of being summarized in the table format.} the existing upper (sufficient) and lower (worst case necessary) bounds on the size ($|\cI|$, or $\E[|\cI|]$ for randomized algorithms) of intervention sets that fully orient a given essential graph.
These lower bounds are ``worst case'' in the sense that there exists a graph, typically a clique, which requires the stated number of interventions.
Observe that there are settings where adaptivity\footnote{Given an essential graph $\cE(G^*)$, non-adaptive algorithms decide a \emph{set} of interventions without looking at the outcomes of the interventions.
Meanwhile, adaptive algorithms can provide a \emph{sequence} of interventions one-at-a-time, possibly using any information gained from the outcomes of earlier chosen interventions.} and randomization strictly improves the number of required interventions.

\textbf{Separating systems}
\cite{hyttinen2013experiment} drew connections between causal discovery via interventions and the concept of separating systems from the combinatorics literature.
This was extended by \cite{shanmugam2015learning} to the bounded size and adaptive settings.
An $(n,k)$-separating system is a Boolean matrix with $n$ columns where each row has at most $k$ ones, indicating which vertex is to be intervened upon.
Using their proposed separating system construction based on ``label indexing'', \cite{shanmugam2015learning} showed that roughly $\frac{n}{k} \log_{\frac{n}{k}} n$ interventions is sufficient to fully an essential graph $G$ with bounded size interventions.
On cliques (i.e.\ worst case lower bound), \cite{shanmugam2015learning} showed that the bound is tight while only roughly $\frac{\chi(\cE(G))}{k} \log_{\frac{\chi(\cE(G))}{k}} \chi(\cE(G))$ interventions are necessary for general graphs\footnote{Note that there is a slight gap between $\frac{\chi(\cE(G))}{k} \log_{\frac{\chi(\cE(G))}{k}} \chi(\cE(G))$ and $\frac{n}{k} \log_{\frac{n}{k}} n$ on general graphs.}, even if the interventions are chosen adaptively or in a randomized fashion.

\textbf{Universal bounds for minimum sized atomic interventions}
Beyond worst case lower bounds, recent works have studied universal bounds for orienting essential graphs $\cE(G^*)$ using atomic interventions \cite{squires2020active,porwal2021almost}.
These universal bounds depend on graph parameters of $\cE(G)$ beyond the number of nodes $n$.
\cite{squires2020active} showed that search algorithms must use at least $\sum_{H \in CC(\cE(G^*))} \lfloor \frac{\omega(H)}{2} \rfloor$ interventions, where $H$ is a chain component of $\cE(G^*)$ and the summation across chain components is a consequence of \cref{lem:hauser-bulmann-strengthened}.
They also introduced a graph concept called directed clique trees and designed an adaptive, deterministic algorithm.
On intersection-incomparable chordal graphs, their algorithm outputs an intervention set of size $\cO(\log_2 (\max_{H \in CC(\cE(G^*))} \omega(H)) \cdot \nu_1(G^*))$.
More recently, \cite{porwal2021almost} introduced the notion of clique-block shared-parents orderings and showed that any search algorithm for an essential graph $\cE(G^*)$ with $r$ maximal cliques requires at least $\lceil \frac{n-r}{2} \rceil$ interventions and $\nu_1(G) \leq n-r$ for any $G \in [G^*]$.

\textbf{Non-atomic interventions}
The randomized algorithm of \cite{hu2014randomized} fully orients an essential graph using $\cO(\log(\log(n)))$ unbounded interventions in expectation.
Building upon this, \cite{shanmugam2015learning} shows that $\cO(\frac{n}{k} \log(\log(k)))$ bounded sized interventions (each involving at most $k$ nodes) suffice.

\textbf{Additive vertex costs}
\cite{kocaoglu2017cost,ghassami2018budgeted,lindgren2018experimental} studied the \emph{non-adaptive} search setting where vertices may have different intervention costs and intervention costs accumulate additively.
\cite{ghassami2018budgeted} studied the problem of maximizing number of oriented edges given a budget of atomic interventions while \cite{kocaoglu2017cost,lindgren2018experimental} studied the problem of finding a minimum cost (bounded size) intervention set that fully orients the essential graph.
\cite{lindgren2018experimental} showed that computing the minimum cost intervention set is NP-hard and gave search algorithms with constant approximation factors.

\textbf{Other related work}
\cite{hu2014randomized,katz2019size} showed that Erd\H{o}s-R\'{e}nyi graphs can be easily oriented.
\section{Results}
\label{sec:results}

\subsection{Verification}

Our core contribution for the verification problem is deriving an interesting connection between the covered edges and verifying sets.
We show importance of this connection by using it to derive several novel results on finding optimal verifying sets in various settings such as bounded size interventions and when vertices have varying interventional costs.
For detailed proofs, see \cref{sec:appendix-verification}.

\begin{restatable}{theorem}{optimalverifyingset}
\label{thm:optimal-verifying-set}
Fix an essential graph $\cE(G^*)$ and $G \in [G^*]$.
An intervention set $\cI$ is a verifying set for $G$ if and only if $\cI$ is a set that separates every covered edge of $G$ that is unoriented in $\cE(G^*)$.
\end{restatable}

Together with \cref{lem:sequence} (any undirected edge in $\cE(G^*)$ is a covered edge for \emph{some} $G \in [G^*]$), \cref{thm:optimal-verifying-set} implies a simple alternative proof for an earlier known result that characterizes non-adaptive search algorithms via separating systems \cite{hyttinen2013experiment,shanmugam2015learning}: any non-adaptive search algorithm, which has \emph{no} knowledge of $G^*$, should separate \emph{every} undirected edge in $\cE(G^*)$.

Another immediate application of \cref{thm:optimal-verifying-set} is the following result for the verification problem.

\begin{restatable}{corollary}{solvingverification}
\label{cor:solving-verification}
Given an essential graph $\cE(G^*)$ of an unknown ground truth DAG $G^*$ and a causal DAG $G \in [G^*]$, we can test if $G \stackrel{?}{=} G^*$ by intervening on any verifying set of $G$.
Furthermore, in the worst case, \emph{any} algorithm that correctly resolves $G \stackrel{?}{=} G^*$ needs at least $\nu(G)$ interventions.
\end{restatable}

The above corollary provides a solution to the verification problem in terms of verifying sets and in the following we give efficient algorithms for computing these verifying sets that are optimal in the atomic setting and are near-optimal in the case of bounded size.

\begin{restatable}{theorem}{efficientoptimalatomic}
\label{thm:efficient-optimal-atomic}
Fix an essential graph $\cE(G^*)$ and $G \in [G^*]$.
An atomic intervention set $\cI$ is a minimal sized verifying set for $G$ if and only if $\cI$ is a minimum vertex cover of unoriented covered edges of $G$.
A minimal sized atomic verifying set can be computed in polynomial time.
\end{restatable}

Our result provides the first efficient algorithm for computing minimum sized atomic verifying set for general graphs.
Previously, efficient algorithms for computing minimum sized atomic verifying sets were only known for simple graphs such as cliques and trees. For general graphs, only a brute force algorithm is known \cite[Appendix F]{squires2020active} which takes exponential time in the worst case\footnote{\cref{thm:efficient-optimal-atomic} also provides a rigorous justification to the observation of \cite{squires2020active} that ``In general, the size of an [atomic verifying set] cannot be calculated from just its essential graph''. This is because essential graphs could imply minimum vertex covers of different sizes (see \cref{fig:standing-windmill}).}.
In contrast to optimal algorithms, \cite{porwal2021almost} provides an efficient algorithm that returns a verifying set of size at most 2 times that of the optimum. 

\begin{restatable}{theorem}{efficientnearoptimalbounded}
\label{thm:efficient-near-optimal-bounded}
Fix an essential graph $\cE(G^*)$ and $G \in [G^*]$.
If $\nu_1(G) = \ell$, then $\nu_k(G) \geq \lceil \frac{\ell}{k} \rceil$ and there exists a polynomial time algo.\ to compute a bounded size intervention set $\cI$ of size $|\cI| \leq \lceil \frac{\ell}{k} \rceil + 1$.
\end{restatable}

To the best of our knowledge, our work provides the first known efficient algorithm for computing near-optimal bounded sized verifying sets for general graphs.
Furthermore, we note that, for every $k$, there exists a family of graphs where the optimum solution requires at least $\lceil \frac{\ell}{k} \rceil + 1$ bounded size interventions.
Thus, our upper bound is tight in the worst case (see \cref{fig:near-optimal} in \cref{sec:near-optimal-bounded}).

Beyond minimal sized interventions, a natural and much broader setting in causal inference is one where different vertices have varying intervention costs, e.g.\ in a smoking study, it is easier to modify a subject's diet than to force the subject to smoke (or stop smoking).
Formally, one can define a weight function on the vertices $w: V \to \R$ which overloads to $w(S) = \sum_{v \in S} w(v)$ on interventions and $w(\cI) = \sum_{S \in \cI} S$ on intervention sets.
Such an additive cost structure has been studied by \cite{kocaoglu2017cost,ghassami2018budgeted}.
Consider an essential graph which is a star graph on $n$ nodes where the leaves have cost 1 and the root has cost significantly larger than $n$.
For atomic verifying sets, we see that the \emph{minimum cost} verifying set is to intervene on the leaves while the \emph{minimum size} verifying set is to simply intervene on the root.
Since one may be more preferred over the other, depending on the actual real-life situation, we propose to find a verifying set $\cI$ which minimizes
\begin{equation}
\label{eq:generalized-cost}
\alpha \cdot w(\cI) + \beta \cdot |\cI| \qquad \text{where $\alpha, \beta \geq 0$}
\end{equation}
so as to explicitly trade-off between the cost and size of the intervention set.
This objective also naturally allows the constraint of bounded size interventions by restricting $|S| \leq k$ for all $S \in \cI$.
Note that the earlier results on minimum size verifying sets correspond to $\alpha = 0$ and $\beta = 1$.
Interestingly, the techniques we developed in the minimum size setting generalizes to this broader setting.
Furthermore, our framework of covered edges allow us to get simple proofs.

\begin{restatable}{theorem}{efficientoptimalatomiccost}
\label{thm:efficient-optimal-atomic-cost}
Fix an essential graph $\cE(G^*)$ and $G \in [G^*]$.
An atomic verifying set $\cI$ for $G$ that minimizes \cref{eq:generalized-cost} can be computed in polynomial time.
\end{restatable}

\begin{restatable}{theorem}{efficientnearoptimalboundedcost}
\label{thm:efficient-near-optimal-bounded-cost}
Fix an essential graph $\cE(G^*)$ and $G \in [G^*]$.
Suppose the optimal bounded size intervention set that minimizes \cref{eq:generalized-cost} costs $OPT$.
Then, there exists a polynomial time algorithm that computes a bounded size intervention set with total cost $OPT + 2 \beta$.
\end{restatable}

\subsection{Adaptive search}

Here, we study the unweighted search problem in causal inference where one wishes to fully orient an essential graph obtained from observational data while minimizing the number of interventions used.
Formally, we give an algorithm (\cref{alg:search-algo}) that fully orients $\cE(G^*)$ using at most a logarithmic multiplicative factor more interventions than $\nu_k(G^*)$, the number of (bounded size) interventions needed to \emph{verify} the ground truth $G^*$.
Since \emph{any} search algorithm will incur at least $\nu_k(G^*)$ interventions, our result implies that search is (almost, up to $\log n$ multiplicative approximation) as easy as the verification.
For detailed proofs, see \cref{sec:appendix-search}.

\begin{restatable}{theorem}{searchatomic}
\label{thm:search-atomic}
Fix an essential graph $\cE(G^*)$ with an unknown underlying ground truth DAG $G^*$.
Given $k=1$, \cref{alg:search-algo} runs in polynomial time and computes an atomic intervention set $\cI$ in a deterministic and adaptive manner such that $\cE_{\cI}(G^*) = G^*$ and $|\cI| \in \cO(\log(n) \cdot \nu_1(G^*))$.
\end{restatable}

\begin{restatable}{theorem}{searchbounded}
\label{thm:search-bounded}
Fix an essential graph $\cE(G^*)$ with an unknown underlying ground truth DAG $G^*$.
Given $k > 1$, \cref{alg:search-algo} runs in polynomial time and computes a bounded size intervention set $\cI$ in a deterministic and adaptive manner such that $\cE_{\cI}(G^*) = G^*$ and $|\cI| \in \cO(\log(n) \cdot \log (k) \cdot \nu_k(G^*))$.
\end{restatable}

These results are the first competitive results that holds for using atomic or bounded size interventions on \emph{general graphs}.
The only previously known result of $\cO(\log_2 (\max_{H \in CC(\cE(G^*))} \omega(H)) \cdot \nu_1(G^*))$ by \cite{squires2020active} was an algorithm based on directed clique trees with provable guarantees only for atomic interventions on intersection-incomparable chordal graphs.
To obtain our results, we are \emph{not} simply improving the analysis of \cite{squires2020active}.
Instead, we developed a new algorithmic approach that is based on graph separators which is a much simpler concept than directed clique trees.

The approximation of $\cO(\log n)$ to $\nu_1(G^*)$ is the tightest one can hope for atomic interventions in general.
For instance, consider the case where $\cE(G^*)$ is an undirected line graph on $n$ vertices.
Then, \emph{any} adaptive algorithm needs $\Omega(\log n)$ atomic interventions in the worst case\footnote{The lower bound reasoning is similar to the lower bound for binary search.} while $\nu_1(G^*) = 1$.
The line graph also provides a clear distinction between adaptive and non-adaptive search algorithms since \emph{any} non-adaptive algorithm needs $\Omega(n)$ atomic interventions to separate all the edges in $\cE(G^*)$.
\section{Overview of techniques}
\label{sec:techniques}

\subsection{Verification}

Our results for the verification problem are broadly divided into two categories:
(I) Connection between the covered edges and verification set (\cref{thm:optimal-verifying-set}, \cref{cor:solving-verification});
(II) Efficient computation of verification sets under various settings using the connection established.

For the first type of results, we show that any intervention set is a verifying set if and only if it separates every unoriented covered edge of $G \in \cE(G^*)$.
For necessity, we show that all four Meek rules (which are known to be consistent and complete) will \emph{not} orient any unoriented covered edge of $G$ that is \emph{not} separated by any intervention.
Our proof is simple due to the usage of covered edges.
For sufficiency, we show that \emph{every} unoriented non-covered edge of $G$ will be oriented by Meek rules if all covered edges are separated.
We prove this using a subtle  induction over a valid topological ordering of the vertices $\pi$ of $G^*$:
Let $V_i$ be the first $i$ smallest vertices in $\pi$, for $i = 1, 2, \ldots, n$.
Consider subgraph $\cE(G^*)[V_i]$ induced by $V_i$ with $v_i$ being the last vertex in the ordering of $V_i$.
By induction, it suffices to show that all non-covered $u \to v_i$ edges are oriented for $u \in V_{i-1}$.
To show this, we perform case analysis to argue that either $u \to v_i$ is part of v-structure (i.e.\ is already oriented in $\cE(G^*)$) or Meek rule R2 will orient it.

The previous result only establishes equivalence between the verifying sets and the intervention sets that separate unoriented covered edges.
For efficient computation of optimal verifying sets, we prove several additional properties of covered edges, which may be of independent interest.

\begin{restatable}[Properties of covered edges]{lemma}{coverededgesproperties}
\label{lem:covered-edges-properties}
\hspace{0pt}
\begin{enumerate}
    \item Let $H$ be the edge-subgraph induced by covered edges of a DAG $G$.
    Then, every vertex in $H$ has at most one incoming edge and thus $H$ is a forest of directed trees.
    \item If a DAG $G$ is a clique on $n \geq 3$ vertices $v_1, v_2, \ldots, v_n$ with $\pi(v_1) < \pi(v_2) < \ldots < \pi(v_n)$, then $v_1 \to v_2, \ldots, v_{n-1} \to v_n$ are the covered edges of $G$.
    \item If $u \to v$ is a covered edge in a DAG $G$, then $u$ \emph{cannot} be a sink of any maximal clique of $G$.
\end{enumerate}
\end{restatable}

The fact that the edge-induced subgraph of the covered edges is a forest enables us to use standard dynamic programming techniques to compute (weighted) minimum vertex covers for the unoriented covered edges of $G$, corresponding to minimum size atomic verifying sets (\cref{thm:efficient-optimal-atomic} and \cref{thm:efficient-optimal-atomic-cost}).
For bounded size verifying sets, we exploit the fact that trees are bipartite and so we can divide the minimum vertex covers into two partitions.
Since vertices within each partite are non-adjacent, we can group them into larger interventions without affecting the overall number of separated edges, giving us the guarantees in \cref{thm:efficient-near-optimal-bounded} and \cref{thm:efficient-near-optimal-bounded-cost}.

Through the lens of covered edges, we see that existing universal bounds of \cite{squires2020active,porwal2021almost} are \emph{not} tight\footnote{The non-tightness of the universal lower bound of [PSS22] was known and verified via random graph experiments.}.
Consider the case where the essential graph $\cE(G^*)$ is the standing windmill graph given in \cref{fig:standing-windmill}.
The graph $\cE(G^*)$ has $n = 8$ nodes, $r = 4$ maximal cliques and the largest maximal clique is size $3$.
The lower bound of \cite{squires2020active} yields $\sum_{H \in CC(\cE(G^*))} \lfloor \frac{\omega(H)}{2} \rfloor = \lfloor \frac{3}{2} \rfloor = 1$ while lower bound of \cite{porwal2021almost} yields $\lceil \frac{n-r}{2} \rceil = \lceil \frac{8-4}{2} \rceil = 2$.
Meanwhile, we show $\min_{G \in [G^*]} \nu_1(G) = 3$ in \cref{sec:appendix-standing-windmill}.

We can also recover the $\nu_1(G^*) \leq n-r$ bound of \cite{porwal2021almost} with a short proof using covered edges.

\begin{lemma}
\label{lem:n-r-suffices}
For any essential graph $\cE(G^*)$ on $n$ vertices with $r$ maximal cliques, there exists an atomic verifying set of size at most $n-r$.
\end{lemma}
\begin{proof}
By \cref{thm:efficient-optimal-atomic}, it suffices to find a vertex cover of the unoriented covered edges of $\cE(G^*)$.
By \cref{lem:covered-edges-properties}, any covered edge $u \to v$ cannot have $u$ as a sink of any maximal clique.
So, the set of all vertices without the $r$ sink vertices is a vertex cover of the covered edges in $\cE(G^*)$.
\end{proof}

\subsection{Search}

Here we provide the description and proof overview of the search results.
Our search algorithm (\cref{alg:search-algo}) relies on graph separators that we formally define next.
Existence and efficient computation of graph separators are well studied \cite{lipton1979separator,gilbert1984separatorboundedgenus,gilbert1984separatorchordal,alon1990separator,kawarabayashi2010separator,wulff2011separator} and are commonly used in divide-and-conquer graph algorithms and as analysis tools.

\begin{definition}[$\alpha$-separator and $\alpha$-clique separator]
Let $A,B,C$ be a partition of the vertices $V$ of a graph $G = (V,E)$.
We say that $C$ is an \emph{$\alpha$-separator} if no edge joins a vertex in $A$ with a vertex in $B$ and $|A|, |B| \leq \alpha \cdot |V|$. We call $C$ is an \emph{$\alpha$-clique separator} if it is an \emph{$\alpha$-separator} and a clique.
\end{definition}

\begin{algorithm}[htbp]
\caption{Search algorithm via graph separators.}
\label{alg:search-algo}
\begin{algorithmic}[1]
    \State \textbf{Input}: Essential graph $\cE(G^*)$, intervention size $k$. \textbf{Output}: A fully oriented graph $G \in [G^*]$.
    \State Initialize $i=0$ and $\cI_0 = \emptyset$.
	\While{$\cE_{\cI_{i}}(G^*)$ still has undirected edges}
    	\State For each $H \in CC(\cE_{\cI_{i}}(G^*))$ of size $|H| \geq 2$, find a 1/2-clique separator $K_H$ using
    	\Statex\hspace{\algorithmicindent}\cref{thm:chordal-separator}.
    	Define $Q = \{K_H\}_{H \in CC(\cE_{\cI_{i}}(G^*)), |H| \geq 2}$ as the union of clique separator nodes.
    	\State \textbf{if} $k=1$ or $|Q|=1$ \textbf{then} Define $C_i = Q$ as an atomic intervention set.
    	\State \textbf{else} Define $k' = \min\{k, |Q|/2\}$, $a = \lceil |Q|/k' \rceil \geq 2$, and $\ell = \lceil \log_a n \rceil$. Compute labelling
    	\Statex\hspace{\algorithmicindent}scheme of \cite[Lemma 1]{shanmugam2015learning} on $Q$ with $(|Q|, k', a)$, and define $C_i = \{S_{x,y}\}_{x \in [\ell], y \in [a]}$,
    	\Statex\hspace{\algorithmicindent}where $S_{x,y} \subseteq Q$ is the subset of vertices whose $x^{th}$ letter in the label is $y$.
    	\State Update $i \gets i+1$, intervene on $C_{i}$ to obtain $\cE_{\cI_{i}}(G^*)$, and update $\cI_i \gets \cI_{i-1} \cup C_{i}$.
	\EndWhile
\end{algorithmic}
\end{algorithm}

We first give the proof strategy for \cref{thm:search-atomic}, the atomic case where $k=1$.
We divide the analysis into two steps.
In the first step, we show that the algorithm terminates in $\cO(\log n)$ iterations.
In the second step, we argue that the number of interventions performed in each iteration is in $\cO(\nu_{1}(G))$.

The analysis of the first step relies on the following result of \cite{gilbert1984separatorchordal}.

\begin{theorem}[\cite{gilbert1984separatorchordal}, instantiated for unweighted graphs]
\label{thm:chordal-separator}
Let $G = (V,E)$ be a chordal graph with $|V| \geq 2$ and $p$ vertices in its largest clique.
There exists a $1/2$-clique-separator $C$ of size $|C| \leq p-1$.
The clique $C$ can be computed in $\cO(|E|)$ time.
\end{theorem}

At each iteration $i$ of the algorithm, we intervene on a $1/2$-clique separators of each connected chain component.
Note that the size of each connected chain compoenent at the end of iteration $i$ is at most $n/2^i$.
Therefore, after $\cO(\log n)$ iterations, each connected chain component will contain at most $1$ vertex, implying that all the edges have been oriented. 

To bound the number of interventions used in each iteration, we prove a stronger universal lower bound that is built upon the lower bound of \cite{squires2020active}.
While \emph{not} computable\footnote{It involves a maximization over all possible atomic interventions \emph{and} we do not know the $\cE_{\cI}(G^*)$'s.}, it is a very powerful lower bound for analysis:
In \cref{fig:chain-triangles} of \cref{sec:appendix-chain-triangles}, we give an example where $\nu_1(G^*) \approx n$ while the lower bound of \cite{squires2020active} on $CC(\cE(G^*))$ is a constant.
Meanwhile, there exists a set of atomic interventions $\cI$ such that applying \cite{squires2020active} on $CC(\cE_{\cI}(G^*))$ yields a much stronger $\Omega(n)$ bound.

\begin{restatable}{lemma}{strengthenedlb}
\label{lem:strengthened-lb}
Fix an essential graph $\cE(G^*)$ and $G \in [G^*]$.
Then,
\[
\nu_1(G) \geq \max_{\cI \subseteq V} \sum_{H \in CC(\cE_{\cI}(G^*))} \left\lfloor \frac{\omega(H)}{2} \right\rfloor
\]
\end{restatable}

As we intervene on cliques in each connected component at every iteration, \cref{lem:strengthened-lb} shows that we use at most $2 \cdot \nu_1(G^*)$ interventions per iteration.
Therefore, the total interventions used by \cref{alg:search-algo} is in $\cO(\log (n) \cdot \nu_1(G^*))$.

For bounded size interventions (\cref{thm:search-bounded}), we follow the same strategy as above but we modify how we orient the edges within and cut by the $1/2$-clique separators.
More specifically, we compute a separating system for the union of clique separator nodes based on the labeling scheme of \cite{shanmugam2015learning} and perform case analysis to argue that we use $\cO(\log (k) \cdot \nu_k(G^*))$ bounded size interventions per iteration.

\begin{lemma}[Lemma 1 of \cite{shanmugam2015learning}]
Let $(n,k,a)$ be parameters where $k \leq n/2$.
There is a polynomial time labeling scheme that produces distinct $\ell$ length labels for all elements in $[n]$ using letters from the integer alphabet $\{0\} \cup [a]$ where $\ell = \lceil \log_a n \rceil$.
Further, in every digit (or position), any integer letter is used at most $\lceil n/a \rceil$ times.
This labelling scheme is a separating system: for any $i,j \in [n]$, there exists some digit $d \in [\ell]$ where the labels of $i$ and $j$ differ.
\end{lemma}

\section{Experiments and implementation}
\label{sec:experiments}

We implement our verification algorithm and test its correctness on some well-known graphs such as cliques and trees, for which we know the exact verification number.
In addition, we implement \cref{alg:search-algo} and compare its performance with other known atomic search algorithms \cite{he2008active,hauser2014two,shanmugam2015learning,squires2020active} via the experimental setup of \cite{squires2020active}: on synthetic graphs of varying sizes, we compare the runtime and total number of interventions performed compared to the verification number of the underlying DAG.
In \cref{sec:appendix-experiments}, we provide the full experimental details and results of running various search algorithms (including ours) on different graphs.
Qualitatively, our algorithm is competitive with the state-of-the-art search algorithms while being $\sim$10x faster in some experiments.
\cref{fig:exp-section} shows a subset of these results.
We also investigated the impact of different $k$ values on the performance of our search algorithm in \cref{sec:appendix-experiments}. The implementations, along with entire experimental setup, are available at \url{https://github.com/cxjdavin/verification-and-search-algorithms-for-causal-DAGs}.

\begin{figure}[htbp]
\centering
\begin{subfigure}{0.49\textwidth}
    \centering
    \includegraphics[width=\linewidth]{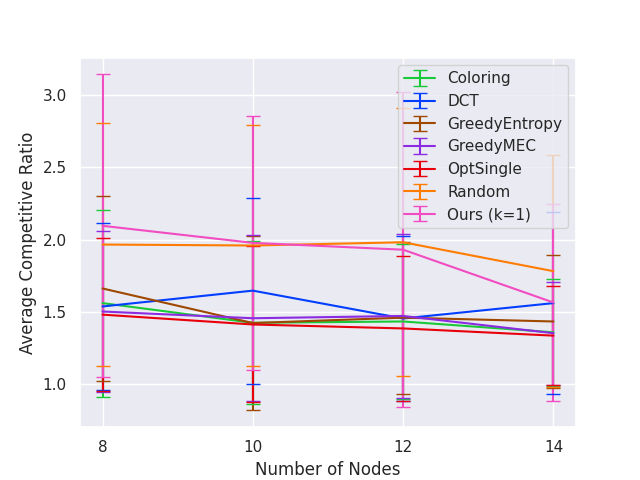}
\end{subfigure}
\begin{subfigure}{0.49\textwidth}
    \centering
    \includegraphics[width=\linewidth]{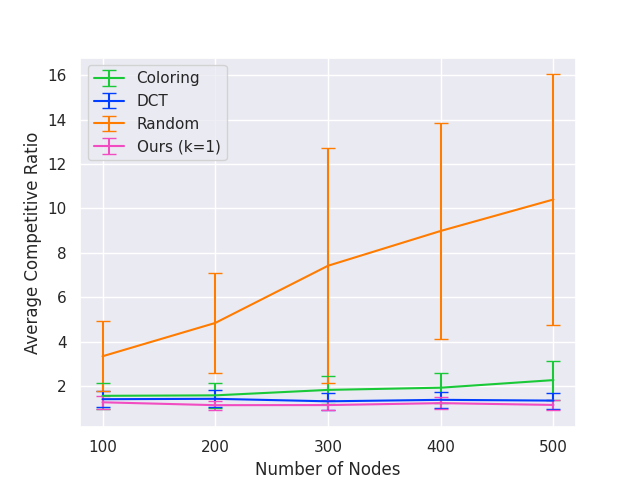}
\end{subfigure}
\caption{
The figures show the average competitive ratios with respect to synthetic graphs of different node sizes.
See \cref{sec:appendix-experiments} for details about how the synthetic graphs are generated and for details on the algorithms benchmarked.
We also show the \emph{maximum} competitive ratios in \cref{sec:appendix-experiments}.
}
\label{fig:exp-section}
\end{figure}

\textbf{Conclusion, limitations, and societal impact}
\label{sec:statement}
Learning causal relationships is of fundamental importance to science and society in general.
This is especially important when one wishes to correctly predict effects of making changes to a system for downstream tasks such as designing fair algorithms.
In this work, we gave a complete understanding of the verification problem and an improved search algorithm under the standard causal inference assumptions (see \cref{sec:introduction}).
However, if our assumptions are violated by the data, then wrong causal conclusions may be drawn and possibly lead to unintended downstream consequences.
Hence, it is of great interest to remove/weaken these assumptions while maintaining strong theoretical guarantees.
A crucial limitation of this work is that we study an idealized setting with hard interventions and infinite samples while soft interventions may be more realistic in certain real-life scenarios (e.g.\ effects from parental vertices are not completely removed but only altered) and sample complexities play a crucial role when one has limited experimental budget (e.g.\ see \cite{kocaoglu2019characterization} and \cite{acharya2018learning} respectively).

\subsubsection*{Acknowledgements}
This research/project is supported by the National Research Foundation, Singapore under its AI Singapore Programme (AISG Award No: AISG-PhD/2021-08-013).
KS was supported by a Stanford Data Science Scholarship, a Dantzig-Lieberman Research Fellowship and a Simons-Berkeley Research
Fellowship.
AB is partially supported by an NUS Startup Grant (R-252-000-A33-133) and an NRF Fellowship for AI (NRFFAI1-2019-0002).
We would like to thank Themis Gouleakis, Dimitrios Myrisiotis, and Chandler Squires for valuable feedback and discussions.
Part of this work was done while the authors were visiting the Simons Institute for the Theory of Computing.

\bibliography{refs}
\bibliographystyle{alpha}

\newpage
\appendix
\section{Meek rules}
\label{sec:appendix-meek-rules}

Meek rules are a set of 4 edge orientation rules that are sound and complete with respect to any given set of arcs that has a consistent DAG extension \cite{meek1995}.
Given any edge orientation information, one can always repeatedly apply Meek rules till a fixed point to maximize the number of oriented arcs.

\begin{definition}[Consistent extension]
A set of arcs is said to have a \emph{consistent DAG extension} $\pi$ for a graph $G$ if there exists a permutation on the vertices such that (i) every edge $\{u,v\}$ in $G$ is oriented $u \to v$ whenever $\pi(u) < \pi(v)$, (ii) there is no directed cycle, (iii) all the given arcs are present.
\end{definition}

\begin{definition}[The four Meek rules \cite{meek1995}, see \cref{fig:meek-rules} for an illustration]
\hspace{0pt}
\begin{description}
    \item [R1] Edge $\{a,b\} \in E$ is oriented as $a \to b$ if $\exists$ $c \in V$ such that $c \to a$ and $c \not\sim b$.
    \item [R2] Edge $\{a,b\} \in E$ is oriented as $a \to b$ if $\exists$ $c \in V$ such that $a \to c \to b$.
    \item [R3] Edge $\{a,b\} \in E$ is oriented as $a \to b$ if $\exists$ $c,d \in V$ such that $d \sim a \sim c$, $d \to b \gets c$, and $c \not\sim d$.
    \item [R4] Edge $\{a,b\} \in E$ is oriented as $a \to b$ if $\exists$ $c,d \in V$ such that $d \sim a \sim c$, $d \to c \to b$, and $b \not\sim d$.
\end{description}
\end{definition}

\begin{figure}[htbp]
\centering
\resizebox{\linewidth}{!}{%
\begin{tikzpicture}
%
%
\node[draw, circle, inner sep=2pt] at (0,0) (R1a-before) {\small $a$};
\node[draw, circle, inner sep=2pt, right=of R1a-before] (R1b-before) {\small $b$};
\node[draw, circle, inner sep=2pt, above=of R1a-before](R1c-before) {\small $c$};
\draw[thick, -stealth] (R1c-before) -- (R1a-before);
\draw[thick] (R1a-before) -- (R1b-before);

\node[draw, circle, inner sep=2pt] at (3,0) (R1a-after) {\small $a$};
\node[draw, circle, inner sep=2pt, right=of R1a-after] (R1b-after) {\small $b$};
\node[draw, circle, inner sep=2pt, above=of R1a-after](R1c-after) {\small $c$};
\draw[thick, -stealth] (R1c-after) -- (R1a-after);
\draw[thick, -stealth] (R1a-after) -- (R1b-after);

\node[single arrow, draw, minimum height=2em, single arrow head extend=1ex, inner sep=2pt] at (2.2,0.75) (R1arrow) {};
\node[above=5pt of R1arrow] {\footnotesize R1};

%
%
\node[draw, circle, inner sep=2pt] at (6,0) (R2a-before) {\small $a$};
\node[draw, circle, inner sep=2pt, right=of R2a-before] (R2b-before) {\small $b$};
\node[draw, circle, inner sep=2pt, above=of R2a-before](R2c-before) {\small $c$};
\draw[thick, -stealth] (R2a-before) -- (R2c-before);
\draw[thick, -stealth] (R2c-before) -- (R2b-before);
\draw[thick] (R2a-before) -- (R2b-before);

\node[draw, circle, inner sep=2pt] at (9,0) (R2a-after) {\small $a$};
\node[draw, circle, inner sep=2pt, right=of R2a-after] (R2b-after) {\small $b$};
\node[draw, circle, inner sep=2pt, above=of R2a-after](R2c-after) {\small $c$};
\draw[thick, -stealth] (R2a-after) -- (R2c-after);
\draw[thick, -stealth] (R2c-after) -- (R2b-after);
\draw[thick, -stealth] (R2a-after) -- (R2b-after);

\node[single arrow, draw, minimum height=2em, single arrow head extend=1ex, inner sep=2pt] at (8.2,0.75) (R2arrow) {};
\node[above=5pt of R2arrow] {\footnotesize R2};

%
%
\node[draw, circle, inner sep=2pt] at (12,0) (R3d-before) {\small $d$};
\node[draw, circle, inner sep=2pt, above=of R3d-before](R3a-before) {\small $a$};
\node[draw, circle, inner sep=2pt, right=of R3a-before] (R3c-before) {\small $c$};
\node[draw, circle, inner sep=2pt, right=of R3d-before](R3b-before) {\small $b$};
\draw[thick, -stealth] (R3c-before) -- (R3b-before);
\draw[thick, -stealth] (R3d-before) -- (R3b-before);
\draw[thick] (R3c-before) -- (R3a-before) -- (R3d-before);
\draw[thick] (R3a-before) -- (R3b-before);

\node[draw, circle, inner sep=2pt] at (15,0) (R3d-after) {\small $d$};
\node[draw, circle, inner sep=2pt, above=of R3d-after](R3a-after) {\small $a$};
\node[draw, circle, inner sep=2pt, right=of R3a-after] (R3c-after) {\small $c$};
\node[draw, circle, inner sep=2pt, right=of R3d-after](R3b-after) {\small $b$};
\draw[thick, -stealth] (R3c-after) -- (R3b-after);
\draw[thick, -stealth] (R3d-after) -- (R3b-after);
\draw[thick] (R3c-after) -- (R3a-after) -- (R3d-after);
\draw[thick, -stealth] (R3a-after) -- (R3b-after);

\node[single arrow, draw, minimum height=2em, single arrow head extend=1ex, inner sep=2pt] at (14.2,0.75) (R3arrow) {};
\node[above=5pt of R3arrow] {\footnotesize R3};

%
%
\node[draw, circle, inner sep=2pt] at (18,0) (R4a-before) {\small $a$};
\node[draw, circle, inner sep=2pt, above=of R4a-before](R4d-before) {\small $d$};
\node[draw, circle, inner sep=2pt, right=of R4d-before] (R4c-before) {\small $c$};
\node[draw, circle, inner sep=2pt, right=of R4a-before](R4b-before) {\small $b$};
\draw[thick, -stealth] (R4d-before) -- (R4c-before);
\draw[thick, -stealth] (R4c-before) -- (R4b-before);
\draw[thick] (R4d-before) -- (R4a-before) -- (R4c-before);
\draw[thick] (R4a-before) -- (R4b-before);

\node[draw, circle, inner sep=2pt] at (21,0) (R4a-after) {\small $a$};
\node[draw, circle, inner sep=2pt, above=of R4a-after](R4d-after) {\small $d$};
\node[draw, circle, inner sep=2pt, right=of R4d-after] (R4c-after) {\small $c$};
\node[draw, circle, inner sep=2pt, right=of R4a-after](R4b-after) {\small $b$};
\draw[thick, -stealth] (R4d-after) -- (R4c-after);
\draw[thick, -stealth] (R4c-after) -- (R4b-after);
\draw[thick] (R4d-after) -- (R4a-after) -- (R4c-after);
\draw[thick, -stealth] (R4a-after) -- (R4b-after);

\node[single arrow, draw, minimum height=2em, single arrow head extend=1ex, inner sep=2pt] at (20.2,0.75) (R4arrow) {};
\node[above=5pt of R4arrow] {\footnotesize R4};

\draw[thick] (5.25,1.75) -- (5.25,-0.25);
\draw[thick] (11.25,1.75) -- (11.25,-0.25);
\draw[thick] (17.25,1.75) -- (17.25,-0.25);
\end{tikzpicture}
}
\caption{An illustration of the four Meek rules}
\label{fig:meek-rules}
\end{figure}
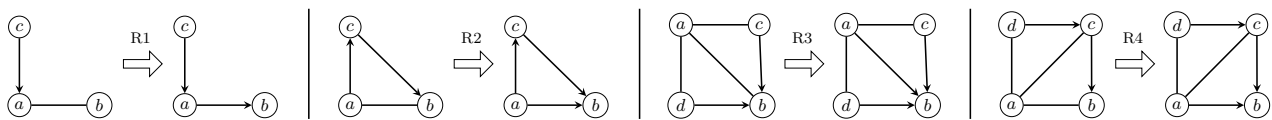

There exists an algorithm \cite[Algorithm 2]{pmlr-v161-wienobst21a} that runs in $\cO(d \cdot |E \cup A|)$ time and computes the closure under Meek rules, where $d$ is the degeneracy of the graph skeleton\footnote{A $d$-degenerate graph is an undirected graph in which every subgraph has a vertex of degree at most $d$. Note that the degeneracy of a graph is typically smaller than the maximum degree of the graph.}.
\section{Proof of \texorpdfstring{\cref{lem:hauser-bulmann-strengthened}}{Lemma 3}}
\label{sec:appendix-stronger-HB}

Lemma 1 of \cite{hauser2014two} actually considers a \emph{single} additional intervention, but a closer look at their proof shows that the statement can be strengthened to allow for \emph{multiple} additional interventions.
In fact, the proof below will almost mimic the proof of \cite[Lemma 1]{hauser2014two} except for some minor changes\footnote{\cite{hauser2014two} considered whether the additional intervention $S \subseteq V$ separates a particular edge. In our proof, we change that argument to whether \emph{some} intervention $S \in \cI'$ separates that same edge. To argue that two graphs $G$ and $H$ are the same, one can show that $G \subseteq H$ and $H \subseteq G$. They only proved ``one direction'' and claim that the other holds by similar arguments. For completeness, we state \emph{exactly} what are changes needed.}.
Note that we can drop the $\emptyset$ intervention in the statement since essential graphs are defined with the observational data provided.
The proof relies on the definition of \emph{strongly protected edges} and a characterization of $\cI$-essential graphs from \cite{hauser2012characterization}.

\begin{definition}[Strong protection; Definition 14 of \cite{hauser2012characterization}]
\label{def:strong-protection}
Let $G = (V,E,A)$ be a (partially oriented) DAG and $\cI \subseteq 2^V$ be an intervention set.
An arc $a \to b$ is \emph{strongly $\cI$-protected} in $G$ if there is some intervention $S \in \cI$ such that $|S \cap \{a,b\}| = 1$, or the arc $a \to b$ occurs in at least one of the following four configurations as an induced subgraph of $G$ (see \cref{fig:strongly-protected}):
\begin{enumerate}
    \item There exists $c \in V$ such that $c \to a \to b$ and $c \not\sim b$.
    \item There exists $c \in V$ such that $a \to b \gets c$ and $c \not\sim a$.
    \item There exists $c \in V$ such that $a \to c \to b$ and $a \to b$.
    \item There exists $c,d \in V$ such that $a \sim c \to b$, $a \sim d \to b$, and $a \to b$.
\end{enumerate}
\end{definition}

\begin{figure}[htbp]
\centering
\resizebox{\linewidth}{!}{%
\begin{tikzpicture}
%
%
\node[draw, circle, inner sep=2pt] at (0,0) (config1-a) {\small $a$};
\node[draw, circle, inner sep=2pt, right=of config1-a] (config1-b) {\small $b$};
\node[draw, circle, inner sep=2pt, above=of config1-a] (config1-c) {\small $c$};
\draw[thick, -stealth] (config1-c) -- (config1-a);
\draw[thick, -stealth] (config1-a) -- (config1-b);

%
%
\node[draw, circle, inner sep=2pt] at (3,0) (config2-a) {\small $a$};
\node[draw, circle, inner sep=2pt, right=of config2-a] (config2-b) {\small $b$};
\node[draw, circle, inner sep=2pt, above=of config2-b] (config2-c) {\small $c$};
\draw[thick, -stealth] (config2-c) -- (config2-b);
\draw[thick, -stealth] (config2-a) -- (config2-b);

%
%
\node[draw, circle, inner sep=2pt] at (6,0) (config3-a) {\small $a$};
\node[draw, circle, inner sep=2pt, right=of config3-a] (config3-b) {\small $b$};
\node[draw, circle, inner sep=2pt, above=of config3-b] (config3-c) {\small $c$};
\draw[thick, -stealth] (config3-c) -- (config3-b);
\draw[thick, -stealth] (config3-a) -- (config3-b);
\draw[thick, -stealth] (config3-a) -- (config3-c);

%
%
\node[draw, circle, inner sep=2pt] at (9,0) (config4-d) {\small $d$};
\node[draw, circle, inner sep=2pt, right=of config4-d] (config4-b) {\small $b$};
\node[draw, circle, inner sep=2pt, above=of config4-d] (config4-a) {\small $a$};
\node[draw, circle, inner sep=2pt, right=of config4-a] (config4-c) {\small $c$};
\draw[thick, -stealth] (config4-c) -- (config4-b);
\draw[thick, -stealth] (config4-d) -- (config4-b);
\draw[thick, -stealth] (config4-a) -- (config4-b);
\draw[thick] (config4-a) -- (config4-c);
\draw[thick] (config4-a) -- (config4-d);
\end{tikzpicture}
}
\caption{An illustration of the four configurations of strongly protected arc $a \to b$}
\label{fig:strongly-protected}
\end{figure}
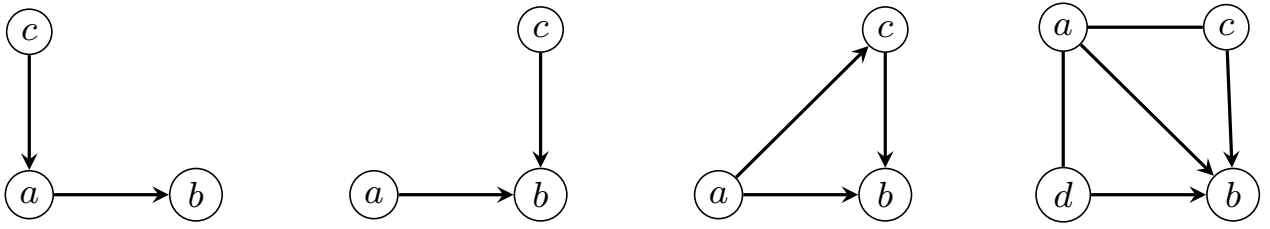

\begin{theorem}[Characterization of $\cI$-essential graphs; Theorem 18 of \cite{hauser2012characterization}]
\label{thm:characterization-of-I-essential-graph}
A (partially oriented) DAG $\cE(G)$ is an $\cI$-essential graph of $G$ if and only if
\begin{enumerate}
    \item $G$ is a chain graph.
    \item For each chain component $H \in CC(G)$, $G[V(H)]$ is chordal.
    \item $G$ has no induced subgraph of the form $a \to b \sim c$.
    \item $G$ has no undirected edge $a \sim b$ whenever $\exists S \in \cI$ such that $|S \cap \{a,b\}| = 1$.
    \item Every arc $a \to b$ in $G$ is strongly $\cI$-protected.
\end{enumerate}
\end{theorem}

For simplicity, we say that an intervention $S \subseteq V$ \emph{separates} an edge $a \sim b$ if $|S \cap \{a,b\}| = 1$ and that an intervention set $\cI \subseteq 2^V$ \emph{separates} an edge $a \sim b$ if it has an intervention that separates it.

\hauserbulmannstrengthened*
\begin{proof}
To shorten notation, define $G = \cE_{\cI \cup \cI'}(G^*)[V(H)]$ and $G' = \cE_{\{S \cap V(H)~:~S \in \cI'\}}(G^*[V(H)])$.
Since $G = (V,E,A)$ and $G' = (V',E',A')$ share the same skeleton (i.e.\ $V = V'$ and $E \cup A = E' \cup A'$) and must respect the same underlying DAG directions of $G^*$, it suffices to argue that $A \subseteq A'$ and $A' \subseteq A$ (i.e.\ they share the same set of directed arcs).
Let $\pi$ be the topological ordering of the ground truth DAG $G^*$.

\textbf{Direction 1 ($A \subseteq A'$)}:
Suppose there are directed arcs in $G$ that are undirected in $G'$.
Let $a \to b$ be one such arc where $\pi(b)$ is \emph{minimized}.

By property 5 of \cref{thm:characterization-of-I-essential-graph}, $a \to b$ is strongly $(\cI \cup \cI')$-protected in $G$.
If $\{S \cap V(H)~:~S \in \cI'\}$ separates $a \sim b$, then $a \to b$ must be oriented in $G'$.
Otherwise, let us consider the 4 configurations given by \cref{def:strong-protection}:
\begin{enumerate}
    \item In $G$, there exists $c \in V$ such that $c \to a \to b$ and $c \not\sim b$.
        By minimality of $\pi(b)$, $c \to a$ must be oriented in $G'$.
        Thus, Meek rule R1 will orient $a \to c$.
    \item In $G$, there exists $c \in V$ such that $a \to b \gets c$ and $c \not\sim a$.
        This is a v-structure in $G^*$ and so $a \to b$ would also be oriented in $G'$.
    \item In $G$, there exists $c \in V$ such that $a \to c \to b$ and $a \to b$.
        By minimality of $\pi(b)$, $a \to c$ must be oriented in $G'$.
        Then, if $a \to b$ is \emph{not} directed in $G$', we will have a directed cycle of the form $a \to c \sim b \sim a$ (regardless of whether the edge $b \sim c$ is directed).
        By property 1 of \cref{thm:characterization-of-I-essential-graph}, $G'$ is a chain graph and \emph{cannot} have such a directed cycle.
        Therefore, $a \to b$ must be oriented in $G'$.
    \item In $G$, there exists $c,d \in V$ such that $a \sim c \to b$, $a \sim d \to b$, and $a \to b$.
        We cannot have $c \to a \gets d$ otherwise such a v-struct will prevent this configuration from occurring.
        Without loss of generality, $a \to c$.
        Then, we can apply the argument of the third configuration on the subgraph induced by $\{a,b,c\}$ to conclude that $a \to b$ is also oriented in $G'$.
\end{enumerate}

\textbf{Direction 2 ($A' \subseteq A$)}:
Repeat the \emph{exact} same argument but perform the following 2 swaps:
\begin{enumerate}
    \item Swap the roles of $G$ and $G'$
    \item Swap the roles of $(\cI \cup \cI')$ and $\{S \cap V(H)~:~S \in \cI'\}$ \qedhere
\end{enumerate}
\end{proof}
\section{Further analysis of the standing windmill essential graph}
\label{sec:appendix-standing-windmill}

In this section, we show that \emph{all} DAGs in the standing windmill essential graph requires at least 3 and at most 4 atomic interventions.

By \cref{thm:efficient-optimal-atomic}, we know that the optimal number of atomic interventions needed to verify any graph is the size of the minimum vertex cover of its oriented edges.
To explore the space of DAGs in the essential graph, we will perform covered edge reversals (as justified by \cref{lem:sequence}).

Consider the DAG $G^*$ with MEC $[G^*]$ and the standing windmill essential graph $\cE(G^*)$ in \cref{fig:standing-windmill-full}.
Starting from $G^*$, if we fix the arc direction $h \to a$, then reversing any arc (possibly multiple times) from the set $\{ b \sim c, d \sim e, f \sim g\}$ does \emph{not} change the covered edge status of any edge (i.e.\ the covered edges remain exactly the same 4 edges) and thus the size of the minimum vertex cover remains unchanged.
Meanwhile, reversing $a \sim h$ in $G^*$ yields the graph $G_1$.
Fixing the arc direction $a \to h$, we observe that the three sets of edges $\{a \sim b, a \sim c, b \sim c\}$, $\{a \sim d, a \sim e, d \sim e\}$, and $\{a \sim f, a \sim g, f \sim g\}$ are symmetric.
Furthermore, if we flip one of the edges from $\{ a \sim b, a \sim d, a \sim f \}$ from $G_1$ (or $\{ a \sim c, a \sim d, a \sim f \}$ from $G_4$), then all other two $a \to \cdot$ arcs are no longer covered edges.
So, it suffices to study what happens when we only reverse arc directions in one of these sets: $\{a \sim b, a \sim c, b \sim c\}$, $\{a \sim d, a \sim e, d \sim e\}$, and $\{a \sim f, a \sim g, f \sim g\}$.
The graphs $G_1$ to $G_6$ illustrate all possible cases when we fix $a \to h$ and only reverse edges in the set $\{a \sim b, a \sim c, b \sim c\}$.
We see that $\nu_1(G^*) = \nu_1(G_1) = \nu_1(G_4) = 4$ and $\nu_1(G_2) = \nu_1(G_3) = \nu_1(G_5) = \nu_1(G_6) = 3$.
Thus, we can conclude that $\min_{G \in [G^*]} \nu_1(G) = 3$ and $\max_{G \in [G^*]} \nu_1(G) = 4$.

\begin{figure}[htbp]
\centering
\input{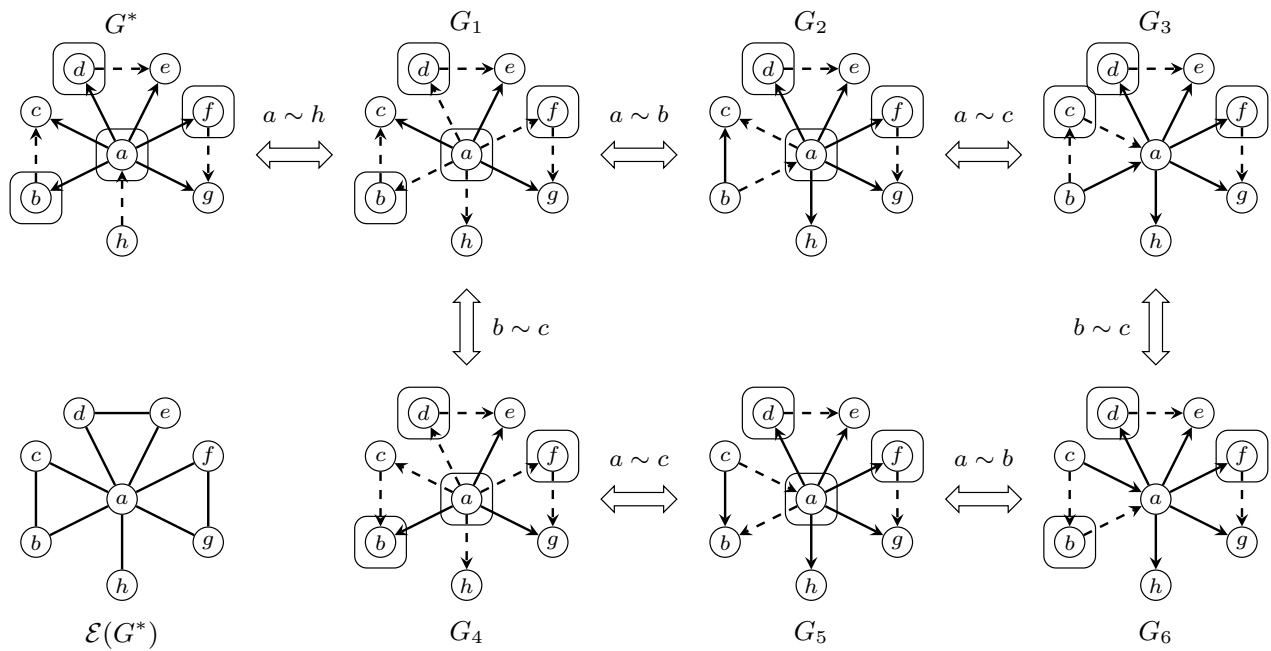}
\caption{
A DAG $G^*$ with its essential graph $\cE(G^*)$ and some of the graphs $G \in [G^*]$.
In each DAG, dashed arcs are covered edges and the boxed vertices represent a minimum vertex cover.
}
\label{fig:standing-windmill-full}
\end{figure}

\section{Upper and (worst case) lower bounds for fully orienting a DAG from its essential graph using ideal interventions}
\label{sec:appendix-landscape}

Let us briefly distinguish the various problem settings before summarizing the state of the art results.

\paragraph{Intervention size}
Since interventions are expensive, natural restrictions on the size of any intervention $S \in \cI$ has been studied.
Bounded size interventions enforce that an upper bound of $|S| \leq k$ always while unbounded size interventions allow $k$ to be as large as $n/2$.
Note that it does not make sense to intervene on a set $S$ with $|S| > n/2$ since intervening on $\overline{S}$ yields the same information while being a strictly smaller interventional set.
Atomic interventions are a special case where $k=1$.

\paragraph{Adaptivity}
A passive/non-adaptive/simultaneous algorithm is one which, given an essential graph $\cE(G^*)$, decides a \emph{set} of interventions without looking at the outcomes of the interventions.
Meanwhile, active/adaptive algorithms can provide a \emph{sequence} of interventions one-at-a-time, possibly using any information gained from the outcomes of earlier chosen interventions.

\paragraph{Determinism}
An algorithm is deterministic if it always produces the same output given the same input.
Meanwhile, randomized algorithms produces an output from a distribution.
Analyses of randomized algorithms typically involve probabilistic arguments and their performance is measured in expectation with probabilistic success\footnote{Typically, they will be shown to succeed with high probability in $n$: as the size of the graph $n$ increases, the failure probability decays quickly in the form of $n^{-c}$ for some constant $c > 1$.}.
The ability to use random bits (e.g.\ outcome of coin flips) is very powerful and may allow one to circumvent known deterministic lower bounds.

\paragraph{Special graph classes}
Two graph classes of particular interest are cliques and trees.
If $CC(\cE(G^*))$ is a clique, then all $\binom{n}{2}$ edges are present and fully orienting the clique is equivalent to finding the unique valid permutation on the vertices.
As such, cliques are often used to prove worst case lower bounds.
Meanwhile, if $CC(\cE(G^*))$ is a tree, then there must be a unique root (else there will be v-structures) and it suffices\footnote{This will later be obvious through the lenses of covered edges: all covered edges are incident to the root.} to intervene on the root node to fully orient the tree.

\cref{tab:upper-bounds} and \cref{tab:lower-bounds} summarize some existing upper (sufficient) and lower (worst case necessary) bounds on the size ($|\cI|$, or $\E[|\cI|]$ for randomized algorithms) of intervention sets that fully orient a given essential graph.
These lower bounds are ``worst case'' in the sense that there exists a graph, typically a clique, which requires the stated number of interventions.
Observe that there are settings where adaptivity and randomization strictly improves the number of required interventions.

\begin{table}[htbp]
\centering
\begin{tabular}{@{}cccccc@{}}
\toprule
Size & Adaptive & Randomized & Graph & Upper bound & Reference\\
\midrule
1 & \xmark & \xmark & General & $n-1$ & \cite{eberhardt2006n}\\
1 & \xmark & \cmark & General & $\frac{2}{3} n - \frac{1}{3}$ for $n > 3$ & \cite{eberhardt2010causal}\\
1 & \cmark & \xmark & Tree & $\cO(\log n)$ & \cite{shanmugam2015learning}\\
1 & \cmark & \xmark & Tree & $\lceil \log n \rceil$ & \cite{greenewald2019sample}\\
$\leq k$ & \xmark & \xmark & General & $(\frac{n}{k} - 1) + \frac{n}{2k}\log_2 k$ & \cite{eberhardt2012number}\\
$\leq k$ & \cmark & \xmark & Tree & $\lceil \log_{k+1} n \rceil$ & \cite{greenewald2019sample}\\
$\leq k$ & \cmark & \cmark & Clique & $\cO(\frac{n}{k} \log \log k)$ & \cite{shanmugam2015learning}\\
$\infty$ & \xmark & \xmark & General & $\log_2 n$ & \cite{eberhardt2012number}\\
$\infty$ & \xmark & \xmark & General & $\lceil \log_2(\omega(\cE(G)) \rceil$ & \cite{hauser2014two}\\
$\infty$ & \xmark & \cmark & General & $\cO(\log \log n)$ & \cite{hu2014randomized}\\
\bottomrule
\end{tabular}
\caption{
Upper bounds on the size ($|\cI|$, or $\E[|\cI|]$ for randomized algorithms) of the intervention set sufficient to fully orient a given essential graph $\cE(G)$.
The first three columns indicate the setting which the algorithm operates in terms of intervention size, adaptivity, and randomness.
The fourth column indicate whether the algorithm is for special graph classes.
Roughly speaking, the algorithm has \emph{more power as we move down the rows} since it can use larger intervention sets, be adaptive, utilize randomization, and possibly only work on special graph classes.
}
\label{tab:upper-bounds}
\end{table}

\begin{table}[htbp]
\centering
\begin{tabular}{@{}ccccc@{}}
\toprule
Size & Adaptive & Randomized & Lower bound & Reference\\
\midrule
1 & \xmark & \cmark & $\frac{2}{3} n - \frac{1}{3}$ for $n > 3$ & \cite{eberhardt2010causal}\\
1 & \cmark & \xmark & $n-1$ & \cite{eberhardt2006n}\\
$\leq k$ & \xmark & \xmark & $(\frac{n}{k} - 1) + \frac{n}{2k}\log_2 k$ & \cite{eberhardt2012number}\\
$\leq k$ & \cmark & \cmark & $\frac{n}{2k}$ & \cite{shanmugam2015learning}\\
$\infty$ & \xmark & \xmark & $\log_2 n$ & \cite{eberhardt2012number}\\
$\infty$ & \xmark & \cmark & $\Omega(\log \log n)$ & \cite{hu2014randomized}\\
$\infty$ & \cmark & \cmark & $\lceil \log_2(\omega(\cE(G)) \rceil$ & \cite{hauser2014two}\\
\bottomrule
\end{tabular}
\caption{
Lower bounds on the size ($|\cI|$, or $\E[|\cI|]$ for randomized algorithms) of the intervention set necessary to fully orient a given essential graph $\cE(G)$.
The first three columns indicate the setting which the algorithm operates in terms of intervention size, adaptivity, and randomness.
Roughly speaking, the setting becomes \emph{easier as we move down the rows} so the lower bounds are \emph{stronger as we move down the rows}.
On cliques, \cite{shanmugam2015learning} also showed that $\geq n/2$ vertices must be intervened.
}
\label{tab:lower-bounds}
\end{table}

\section{Verification}
\label{sec:appendix-verification}

\subsection{Properties of covered edges}

\coverededgesproperties*
\begin{proof}
\hspace{0pt}
\begin{enumerate}
    \item Suppose, for a contradiction, that there exists some vertex $w$ with two incoming covered edges $u \to w \gets v$.
    For $u \to w$ to be covered, we must have $v \to u$.
    Similarly, for $v \to w$ to be covered, we must have $u \to v$.
    However, we cannot simultaneously have both $u \to v$ and $v \to u$, as it would lead to a contradiction as $G$ is a DAG. 
    Furthermore, since $G$ is acyclic, it implies that $H$ must also be acyclic. Therefore $H$ is a forest of directed trees.
    \item Let $A = \{v_1 \to v_2, v_2 \to v_3, \ldots, v_{n-1} \to v_n\}$ be the set of arcs of interest.
    For any arc $v_i \to v_{i+1} \in A$, one can check that they share the same parents by the topological ordering $\pi$.
    Consider an arbitrary arc $v_i \to v_j \not\in A$.
    Since $v_i \to v_j \not\in A$, there exists $v_k \in V$ such that $\pi(v_i) < \pi(v_k) < \pi(v_j)$.
    Then, since $G$ is a clique, we must have $v_i \to v_k \to v_j$ and so $v_i \to v_j$ \emph{cannot} be covered since $v_k \in Pa(v_j) \setminus \{v_i\}$ but $v_k \not\in Pa(v_i) \setminus \{v_j\}$.
    \item Suppose, for a contradiction, that $u$ is a sink of some maximal clique $K_h$ of size $h$ and $u \to v$ is a covered edge.
    Then, we must have $Pa(v) \setminus \{u\} = Pa(u) \setminus v$.
    However, that means that $V(K_h) \cup \{v\}$ is a clique of size $h+1$.
    Thus, $K_h$ was not a maximal clique.
    Contradiction.
\end{enumerate}
\end{proof}

\subsection{Characterization via separation of covered edges}

\begin{lemma}[Necessary]
\label{lem:necessary}
Fix an essential graph $\cE(G^*)$ and $G \in [G^*]$.
If $\cI \subseteq 2^V$ is a verifying set, then $\cI$ separates all unoriented covered edge $u \sim v$ of $G$.
\end{lemma}
\begin{proof}
Let $u \to v$ be an arbitrary unoriented covered edge in $\cE(G^*)$ and $\cI$ be an intervention set where $u$ and $v$ are \emph{never} separated by any $S \in \cI$.
Then, interventions will not orient $u \to v$ and we can only possibly orient it via Meek rules.
We check that all four Meek rules will \emph{not} orient $u \to v$:
\begin{description}
    \item [(R1)] For R1 to trigger, we need to have $w \to u \to v$ and $w \not\sim v$ for some vertex $w \in V \setminus \{u,v\}$.
        However, such a vertex $w$ will imply that $u \to v$ is \emph{not} a covered edge.
    \item [(R2)] For R2 to trigger, we need to have $u \to w \to v$ for some $w \in V \setminus \{u,v\}$.
        However, such a vertex $w$ will imply that $u \to v$ is \emph{not} a covered edge.
    \item [(R3)] For R3 to trigger, we must have $w \sim u \sim x$, $w \to v \gets x$, and $w \not\sim x$ for some $w,x \in V \setminus \{u,v\}$.
        Since $u \to v$ is a covered edge, we must have $w \to u \gets x$.
        This implies that both $w \to u \gets x$ appear as v-structures in $\cE(G^*)$ and thus R3 will not trigger.
    \item [(R4)] For R4 to trigger, we must have $w \sim u \sim x$, $w \to x \to v$, and $w \not\sim v$ for some $w,x \in V \setminus \{u,v\}$.
        Since $u \to v$ is covered, we must have $x \to u$.
        To avoid directed cycles, it must be the case that $w \to u$.
        However, this implies that $u \to v$ is \emph{not} covered since $w \to u$ while $w \not\sim v$.
\end{description}
Therefore, $\cI$ \emph{cannot} be a verifying set if $u$ and $v$ are \emph{never} separated by any $S \in \cI$.
\end{proof}

\begin{lemma}[Sufficient]
\label{lem:sufficient}
Fix an essential graph $\cE(G^*)$ and $G \in [G^*]$.
If $\cI \subseteq 2^V$ is an intervention set that separates every unoriented covered edge $u \sim v$ of $G$, then $\cI$ is a verifying set.
\end{lemma}
\begin{proof}
Let $\cI$ be an arbitrary intervention set such that every unoriented covered edge $u \sim v$ of $G$ has an set $S \in \cI$ that separates $u$ and $v$.
Fix an arbitrary valid vertex permutation $\pi: V \to [n]$ of $G$.
For any $i \in [n]$, define $V_i = \{\pi^{-1}(1), \ldots, \pi^{-1}(i)\} \subseteq V$ as the $i$ smallest vertices according to $\pi$'s ordering.
We argue that any unoriented edges in $\cE(G^*)[V_i]$ will be oriented by $\cI$ by performing induction on $i$.

\textbf{Base case ($i=1$)}: There are no edges in $G[V_1]$ so $\cE(G^*)[V_1]$ is trivially fully oriented.

\textbf{Inductive case ($i > 1$)}: Suppose $v = \pi^{-1}(i)$.
By induction hypothesis, $\cE(G^*)[V_{i-1}]$ is fully oriented so any unoriented edge in $\cE(G^*)[V_i]$ must have the form $u \to v$, where $\pi(u) < \pi(v)$.
For any $u \to v$ is an unoriented covered edge in $\cE(G^*)[V_i]$, there will be an intervention $S \in \cI$ that separates $u$ and $v$ (or both), and hence orient $u \to v$.

Suppose, for a contradiction, that there exists unoriented edges in $\cE(G^*)[V_i]$ that are \emph{not} covered edges.
Let $u \to v$ be the unoriented edge where $\pi(u)$ is \emph{maximized}.
Then, one of the two cases must occur:
\begin{description}
    \item[Case 1] ($u \to v$ and $\exists w \in V_i$ such that $w \to u$ and $w \not\to v$)
        Since $\pi(v) > \pi(w)$, we must have $w \not\sim v$.
        By induction, $w \to u$ will be oriented.
        So, R1 orients $u \to v$.
    \item[Case 2] ($u \to v$ and $\exists w \in V_i$ such that $w \to v$ and $w \not\to u$)
        If $w \not\sim u$, then $u \to v \gets w$ is a v-structure and $u \to v$ would have been oriented.
        If $w \sim u$, then we must have $u \to w$ and $\pi(u) < \pi(w)$.
        By induction, $u \to w$ will be oriented.
        Since $\pi(u) < \pi(w)$ and $\pi(u)$ is maximized out of all possible unoriented edges in $\cE(G^*)[V_i]$ involving $v$, $w \to v$ must be an oriented edge and will be oriented by $\cI$.
        So, R2 orients $u \to v$.
\end{description}
In either case, $u \to v$ will be oriented.
Contradiction.
\end{proof}

Combining \cref{lem:necessary} and \cref{lem:sufficient} gives the following characterization of verifying sets.

\optimalverifyingset*

\subsection{Solving the verification problem}

\solvingverification*
\begin{proof}
Using \cref{thm:optimal-verifying-set}, we know that the minimal verifying set for $G$ is the smallest possible set of interventions $\cI$ such that \emph{all} covered edges of $G$ is separated by some intervention $S \in \cI$.
If the graph is fully oriented after intervening on all $S \in \cI$, then it must be the case that $G = G^*$.
Otherwise, we will either detect that some edge orientation disagrees with $G$ or there remains some unoriented edge at the end of all our interventions.
In the first case, we trivially conclude that $G \neq G^*$.
In the second case, \cref{thm:optimal-verifying-set} tells us that any such unoriented edge must be an unoriented covered edge of $G^*$ (which was not an unoriented cover edge of $G^A$) and so we can also conclude that $G \neq G^*$.

one unoriented covered edge in $G$, reverse it to get $G'$. cannot distinguishi.
\end{proof}

\subsection{Efficient optimal verification via atomic interventions}

\efficientoptimalatomic*
\begin{proof}
For $|S| = 1$, we see that $\cI$ separates every unoriented covered edge in $\cE(G)$ if and only if the set $\cup_{S \in \cI} S$ is a vertex cover of the unoriented covered edges in $\cE(G)$.
\cref{lem:covered-edges-properties} tells us that the edge-induced subgraph on covered edges of $G$ is a forest.
Thus, one can perform the standard dynamic programming algorithm to compute the minimum vertex cover on each tree.
\end{proof}

\subsection{Efficient near-optimal verification via bounded size interventions}
\label{sec:near-optimal-bounded}

We first prove a simple lower bound on the minimum number of non-atomic bounded size interventions (i.e.\ $|S| \leq k$) needed for verification and then show how to adapt a minimal atomic verifying set to obtain a near-optimal bounded size verifying set.

\begin{lemma}
\label{lem:atomic-only-helps}
Fix an essential graph $\cE(G^*)$ and $G \in [G^*]$.
Suppose $\cI$ is an arbitrary bounded size intervention set.
Intervening on vertices in $\cup_{S \in \cI} S$ one at a time, in an atomic fashion, can only increase the number of separated covered edges of $G$.
\end{lemma}
\begin{proof}
Consider an arbitrary covered edge $u \sim v$ that was seprated by some intervention $S \in \cI$.
This means that $|\{u,v\} \cap S| = 1$.
Without loss of generality, suppose $u \in S$.
Then, when we intervene on $u$ in an atomic fashion, we would also separate the edge $u \sim v$.
\end{proof}

\begin{lemma}
\label{lem:bounded-size-lb}
Fix an essential graph $\cE(G^*)$ and $G \in [G^*]$.
If $\nu_1(G) = \ell$, then $\nu_k(G) \geq \lceil \frac{\ell}{k} \rceil$.
\end{lemma}
\begin{proof}
A bounded size intervention set of size strictly less than $\lceil \frac{\ell}{k} \rceil$ involves strictly less than $\ell$ vertices.
By \cref{thm:efficient-optimal-atomic} and \cref{lem:atomic-only-helps}, such an intervention set cannot be a verifying set.
\end{proof}

\begin{lemma}
\label{lem:bounded-size-ub}
Fix an essential graph $\cE(G^*)$ and $G \in [G^*]$.
If $\nu_1(G) = \ell$, then there exists a polynomial time algorithm that computes a bounded size intervention set $\cI$ of size $|\cI| \leq \lceil \frac{\ell}{k} \rceil + 1$.
\end{lemma}
\begin{proof}
Consider the atomic verifying set $\cI$ of $G$.
By \cref{lem:covered-edges-properties}, the edge-induced subgraph on covered edges of $G$ is a forest and is thus 2-colorable.

Split the vertices in $\cI$ into partitions according to the 2-coloring.
By construction, vertices belonging in the same partite will \emph{not} be adjacent and thus choosing them together to be in an intervention $S$ will \emph{not} reduce the number of separated covered edges.
Now, form interventions of size $k$ by greedily picking vertices in $\cI$ within the same partite.
For the remaining unpicked vertices (strictly less than $k$ of them), we form a new intervention with them.
Repeat the same process for the other partite.

This greedy process forms groups of size $k$ and at most 2 groups of sizes, one from each partite.
Suppose that we formed $z$ groups of size $k$ in total and two ``leftover groups'' of sizes $x$ and $y$, where $0 \leq x,y < k$.
Then, $\ell = z \cdot k + x + y$, $\frac{\ell}{k} = z + \frac{x+y}{k}$, and we formed at most $z + 2$ groups.
If $0 \leq x+y < k$, then $\lceil \frac{\ell}{k} \rceil = z+1$.
Otherwise, if $k \leq x+y < 2k$, then $\lceil \frac{\ell}{k} \rceil = z+2$.
In either case, we use at most $\lceil \frac{\ell}{k} \rceil + 1$ interventions, each of size $\leq k$.

One can compute a bounded size intervention set efficiently because the following procedures can all be run in polynomial time:
(i) checking if each edge is a covered edge;
(ii) computing a minimum vertex cover on a tree;
(iii) 2-coloring a tree;
(iv) greedily grouping vertices into sizes $\leq k$.
\end{proof}

\cref{thm:efficient-near-optimal-bounded} follows by combining \cref{lem:bounded-size-lb} and \cref{lem:bounded-size-ub}.

\efficientnearoptimalbounded*

Observe that there exists graphs and values $k$ such that the optimal bounded size verifying set requires at least $\lceil \frac{\ell}{k} \rceil + 1$, and thus our upper bound is tight in the worst case: \cref{fig:near-optimal} shows there exists a family of graphs (and values $k$) such that the optimal bounded size verifying set requires $\lceil \frac{\ell}{k} \rceil + 1$.
However, we do not have a proof that \cref{thm:efficient-near-optimal-bounded} is optimal (or counter example that it is not).

\begin{conjecture}
\label{conj:optimality}
The construction of bounded size verifying set given in \cref{thm:efficient-near-optimal-bounded} is optimal.
\end{conjecture}
 
\begin{figure}[htbp]
\centering
\resizebox{\linewidth}{!}{%
\begin{tikzpicture}
%
%
\node[draw, thick, circle, inner sep=2pt] at (0,0) (root) {};

\node[draw, thick, circle, inner sep=2pt, above left=15pt of root, xshift=-20pt] (topv1) {};
\node[draw, thick, circle, inner sep=2pt, right=10pt of topv1] (topv2) {};
\node[draw, thick, circle, inner sep=2pt, above right=15pt of root, xshift=20pt] (topvlast) {};
\node[draw, thick, circle, inner sep=2pt, left=10pt of topvlast] (topv3) {};
\node[inner sep=2pt] at ($(topv2)!0.5!(topv3)$) {$\dots$};

\node[draw, thick, circle, inner sep=2pt, below=30pt of root] (bot2) {};
\node[draw, thick, circle, inner sep=2pt, below left=25pt of bot2] (bot2v1) {};
\node[draw, thick, circle, inner sep=2pt, right=10pt of bot2v1] (bot2v2) {};
\node[draw, thick, circle, inner sep=2pt, below right=25pt of bot2] (bot2vlast) {};
\node[inner sep=2pt] at ($(bot2v2)!0.5!(bot2vlast)$) {$\dots$};

\node[draw, thick, circle, inner sep=2pt, left=100pt of bot2] (bot1) {};
\node[draw, thick, circle, inner sep=2pt, below left=25pt of bot1] (bot1v1) {};
\node[draw, thick, circle, inner sep=2pt, right=10pt of bot1v1] (bot1v2) {};
\node[draw, thick, circle, inner sep=2pt, below right=25pt of bot1] (bot1vlast) {};
\node[inner sep=2pt] at ($(bot1v2)!0.5!(bot1vlast)$) {$\dots$};

\node[draw, thick, circle, inner sep=2pt, right=100pt of bot2] (bot3) {};
\node[draw, thick, circle, inner sep=2pt, below left=25pt of bot3] (bot3v1) {};
\node[draw, thick, circle, inner sep=2pt, right=10pt of bot3v1] (bot3v2) {};
\node[draw, thick, circle, inner sep=2pt, below right=25pt of bot3] (bot3vlast) {};
\node[inner sep=2pt] at ($(bot3v2)!0.5!(bot3vlast)$) {$\dots$};

\node[inner sep=2pt, yshift=-10pt] at ($(bot2)!0.5!(bot1)$) {$\dots$};
\node[inner sep=2pt, yshift=-10pt] at ($(bot2)!0.5!(bot3)$) {$\dots$};

\draw[thick, -stealth, dashed] (root) -- (topv1);
\draw[thick, -stealth, dashed] (root) -- (topv2);
\draw[thick, -stealth, dashed] (root) -- (topv3);
\draw[thick, -stealth, dashed] (root) -- (topvlast);
\draw[thick, -stealth, dashed] (root) -- (bot1);
\draw[thick, -stealth, dashed] (root) -- (bot2);
\draw[thick, -stealth, dashed] (root) -- (bot3);
\draw[thick, -stealth, dashed] (bot1) -- (bot1v1);
\draw[thick, -stealth, dashed] (bot1) -- (bot1v2);
\draw[thick, -stealth, dashed] (bot1) -- (bot1vlast);
\draw[thick, -stealth, dashed] (bot2) -- (bot2v1);
\draw[thick, -stealth, dashed] (bot2) -- (bot2v2);
\draw[thick, -stealth, dashed] (bot2) -- (bot2vlast);
\draw[thick, -stealth, dashed] (bot3) -- (bot3v1);
\draw[thick, -stealth, dashed] (bot3) -- (bot3v2);
\draw[thick, -stealth, dashed] (bot3) -- (bot3vlast);
\draw[thick, -stealth] (root) to [out=180,in=90] (bot1v1);
\draw[thick, -stealth] (root) to [out=210,in=45] (bot1v2);
\draw[thick, -stealth] (root) to [out=210,in=45] (bot1vlast);
\draw[thick, -stealth] (root) to [out=225,in=90] (bot2v1);
\draw[thick, -stealth] (root) to [out=240,in=110] (bot2v2);
\draw[thick, -stealth] (root) to [out=315,in=90] (bot2vlast);
\draw[thick, -stealth] (root) to [out=330,in=135] (bot3v1);
\draw[thick, -stealth] (root) to [out=330,in=135] (bot3v2);
\draw[thick, -stealth] (root) to [out=0,in=90] (bot3vlast);

\node[fit=(root), rounded corners, draw] {};
\node[fit=(bot1), rounded corners, draw] {};
\node[fit=(bot2), rounded corners, draw] {};
\node[fit=(bot3), rounded corners, draw] {};

%
%
\node[draw, thick, circle, inner sep=2pt] at (8,0) (Hroot) {};
\node[draw, thick, circle, inner sep=2pt, below=30pt of Hroot] (Hv2) {};
\node[draw, thick, circle, inner sep=2pt, left=of Hv2] (Hv1) {};
\node[draw, thick, circle, inner sep=2pt, right=of Hv2] (Hv3) {};
\node[inner sep=2pt] at ($(Hv2)!0.5!(Hv1)$) {$\dots$};
\node[inner sep=2pt] at ($(Hv2)!0.5!(Hv3)$) {$\dots$};

\draw[thick, -stealth, dashed] (Hroot) -- (Hv1);
\draw[thick, -stealth, dashed] (Hroot) -- (Hv2);
\draw[thick, -stealth, dashed] (Hroot) -- (Hv3);

\node[fit=(Hroot), rounded corners, draw] {};
\node[fit=(Hv1), rounded corners, draw] {};
\node[fit=(Hv2), rounded corners, draw] {};
\node[fit=(Hv3), rounded corners, draw] {};

\draw[thick] (5.5,1) -- (5.5,-2);
\node[above=10pt of Hroot] {Induced graph $H$};
\draw [decorate,decoration={brace,amplitude=5pt,mirror,raise=10pt},yshift=-10pt]
(Hv1.west) -- (Hv3.east) node[midway, yshift=-20pt] {\footnotesize $k-1$};
\end{tikzpicture}
}
\caption{
A DAG with its covered edges given in dashed arcs.
The edge-induced subgraph of the covered edges is a tree and the minimum vertex cover is all the non-leaf vertices (the boxed vertices) of size $\ell$.
Denote the graph induced by the boxed vertices by $H$.
Now consider the star graph $H$ on $\ell = k$ nodes with $k-1$ leaves.
All the leaf nodes can be put in the same intervention without affecting the separation of any covered edges.
However, including the root with any of the leaf nodes in a same intervention will cause covered edges to be unseparated.
Thus, using bounded size interventions of size at most $k$, verifying such a DAG requires at least $\lceil \frac{\ell}{k} \rceil + 1 = 2$ interventions.
}
\label{fig:near-optimal}
\end{figure}
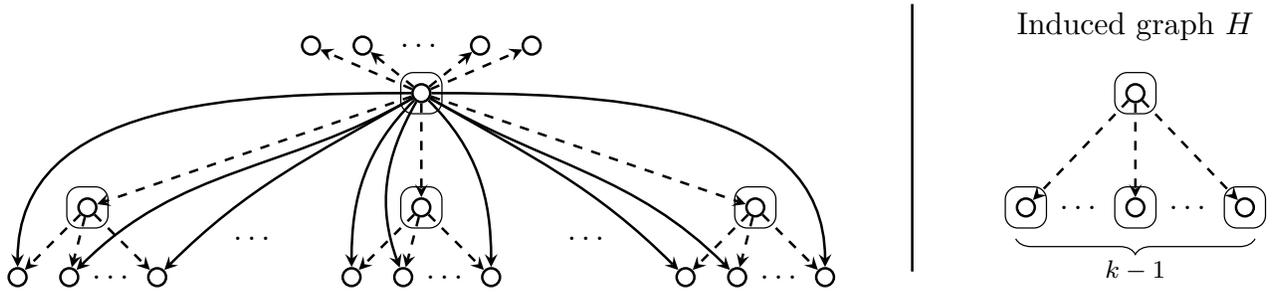

\subsection{Generalization to minimum cost verifying sets with additive cost structures}
\label{sec:generalized-cost}

Consider an essential graph which is a star graph on $n$ nodes where the leaves have cost 1 and the root has cost significantly larger than $n$.
For atomic verifying sets, we see that the \emph{minimum cost} verifying set to intervene on the leaves one at a time while the \emph{minimum size} verifying set is to simply intervene on the root.
Since one may be more preferred over the other, depending on the actual real-life situation, we propose to find a verifying set $\cI$ which minimizes
\begin{equation}
\alpha \cdot w(\cI) + \beta \cdot |\cI| \qquad \text{where $\alpha, \beta \geq 0$}
\tag{1}
\end{equation}
so as to explicitly trade-off between the cost and size of the intervention set.
This objective also naturally allows the constraint of bounded size interventions by restricting $|S| \leq k$ for all $S \in \cI$.

\efficientoptimalatomiccost*
\begin{proof}
By \cref{thm:optimal-verifying-set} and accounting for \cref{eq:generalized-cost}, we need to compute a \emph{weighted} minimum vertex cover in the edge-induced subgraph on covered edges of $G$.
Efficiency is implied by \cref{lem:covered-edges-properties}.
\end{proof}

For bounded size interventions, we show that the ideas in \cref{sec:near-optimal-bounded} translate naturally to give a near-optimal minimal generalized cost verifying set.
To prove our lower bound, we first consider an optimal atomic verifying set for a slightly different objective from \cref{eq:generalized-cost}.

\begin{lemma}
\label{lem:modified-atomic}
Fix an essential graph $\cE(G^*)$ and $G \in [G^*]$.
Let $\cI_A$ be an atomic verifying set for $G$ that minimizes $\alpha \cdot w(\cI_A) + \frac{\beta}{k} \cdot |\cI_A|$ and $\cI_B$ be a bounded size verifying set for $G$ that minimizes \cref{eq:generalized-cost}.
Then, $\alpha \cdot w(\cI_A) + \frac{\beta}{k} \cdot |\cI_A| \leq \alpha \cdot w(\cI_B) + \beta \cdot |\cI_B|$.
\end{lemma}
\begin{proof}
Let $\cI = \sum_{S \in \cI_B} S$ be the atomic verifying set derived from $\cI_B$ by treating each vertex as an atomic intervention.
Clearly, $w(\cI_B) \geq w(\cI)$ and $k \cdot |\cI_B| \geq |\cI|$.
So,
\[
\alpha \cdot w(\cI_B) + \beta \cdot |\cI_B|
\geq \alpha \cdot w(\cI) + \frac{\beta}{k} \cdot |\cI|
\geq \alpha \cdot w(\cI_A) + \frac{\beta}{k} \cdot |\cI_A|
\]
since $\cI_A = \argmin_{\text{atomic verifying set } \cI'} \left\{ \alpha \cdot w(\cI') + \frac{\beta}{k} \cdot |\cI'| \right\}$.
\end{proof}

\efficientnearoptimalboundedcost*
\begin{proof}
Let $\cI_A$ be an atomic verifying set for $G$ that minimizes $\alpha \cdot w(\cI_A) + \frac{\beta}{k} \cdot |\cI_A|$ and $\cI_B$ be a bounded size verifying set for $G$ that minimizes \cref{eq:generalized-cost}.
Using the polynomial time greedy algorithm in \cref{lem:bounded-size-ub}, we construct bounded size intervention set $\cI$ by greedily grouping together atomic interventions from $\cI_A$.
Clearly, $w(\cI) = w(\cI_A)$ and $|\cI| \leq \lceil \frac{|\cI_A|}{k} \rceil + 1$.
So,
\[
\alpha \cdot w(\cI) + \beta \cdot |\cI|
\leq \alpha \cdot w(\cI_A) + \beta \cdot \left( \left\lceil \frac{|\cI_A|}{k} \right\rceil + 1 \right)
\leq \alpha \cdot w(\cI_B) + \beta \cdot |\cI_B| + 2 \beta
= OPT + 2 \beta
\]
where the second inequality is due to \cref{lem:modified-atomic}.
\end{proof}

\section{Search}
\label{sec:appendix-search}

We begin by proving a strengthened version of \cite{squires2020active}'s lower bound.

Note that we will be discussing only atomic interventions in \cref{lem:strengthened-lb}, so notation such as $\cI \cap V(H)$ makes sense for sets $\cI, V(H) \subseteq V$.

\begin{definition}[Moral DAG, Definition 3 of \cite{squires2020active}]
A graph $G$ is a \emph{moral DAG} if its essential graph only has a single chain component.
That is, after removing directed edges, there is only one single connected component remaining.
\end{definition}

\begin{lemma}[Lemma 6 of \cite{squires2020active}]
\label{lem:moral-dag-lemma}
Let $G$ be a moral DAG.
Then, $\nu_1(G) \geq \lfloor \frac{\omega(\text{skeleton}(G))}{2} \rfloor$.
\end{lemma}

\strengthenedlb*
\begin{proof}
Consider be an arbitrary set of atomic interventions $\cI \subseteq V$ and the resulting $\cI$-essential graph $\cE_{\cI}(G^*)$.
Let $H \in CC(\cE_{\cI}(G^*))$ be an arbitrary chain component.

Let $\cI' \subseteq V$ be an arbitrary atomic verifying set of $G$.
Then, $\cE_{\cI'}(G^*) = G$ and thus $\cE_{\cI'}(G^*)[V(H)] = G[V(H)]$.
Then,
\[
\cE_{(\cI' \setminus \cI) \cap V(H)}(G[V(H)])
= \cE_{\cI \cup (\cI' \setminus \cI)}(G)[V(H)]
= \cE_{\cI'}(G)[V(H)]
= G[V(H)]
\]
where the first equality is due to \cref{lem:hauser-bulmann-strengthened} and the last equality is because $\cI'$ is a verifying set of $G$.
So, $(\cI' \setminus \cI) \cap V(H)$ is a verifying set for $G[V(H)]$, and so is $\cI' \cap V(H)$.
Thus, by minimality of $\nu_1$, we have $\nu_1(G[V(H)]) \leq |\cI' \cap V(H)|$ for \emph{any} atomic verifying set $\cI' \subseteq V$ of $G$.

Since $H \in CC(\cE_{\cI}(G^*))$, the graph $G[V(H)]$ is a moral DAG.
Since $H$ is a subgraph of $G[V(H)]$, $\omega(H) \leq \omega(G[V(H)])$.
Thus, by \cref{lem:moral-dag-lemma}, we have $\nu_1(G[V(H)]) \geq \lfloor \frac{\omega(G[V(H)])}{2} \rfloor \geq \lfloor \frac{\omega(H)}{2} \rfloor$.

Now, suppose $\cI^*$ is a minimal size verifying set of $G$.
Then,
\[
\nu_1(G)
= |\cI^*|
\geq \sum_{H \in CC(\cE_{\cI}(G^*))} |\cI^* \cap V(H)|
\geq \sum_{H \in CC(\cE_{\cI}(G^*))} \nu_1(G[V(H)])
\geq \sum_{H \in CC(\cE_{\cI}(G^*))} \left\lfloor \frac{\omega(H)}{2} \right\rfloor
\]

The claim follows by taking the maximum over all possible atomic interventions $\cI \subseteq V$.
\end{proof}

For convenience, we reproduce \cref{alg:search-algo} below.

\setcounter{algorithm}{0}
\begin{algorithm}[htbp]
\caption{Search algorithm via graph separators.}
\begin{algorithmic}[1]
    \State \textbf{Input}: Essential graph $\cE(G^*)$, intervention size $k$. \textbf{Output}: A fully oriented graph $G \in [G^*]$.
    \State Initialize $i=0$ and $\cI_0 = \emptyset$.
	\While{$\cE_{\cI_{i}}(G^*)$ still has undirected edges}
    	\State For each $H \in CC(\cE_{\cI_{i}}(G^*))$ of size $|H| \geq 2$, find a 1/2-clique separator $K_H$ using
    	\Statex\hspace{\algorithmicindent}\cref{thm:chordal-separator}.
    	Define $Q = \{K_H\}_{H \in CC(\cE_{\cI_{i}}(G^*)), |H| \geq 2}$ as the union of clique separator nodes.
    	\State \textbf{if} $k=1$ or $|Q|=1$ \textbf{then} Define $C_i = Q$ as an atomic intervention set.
    	\State \textbf{else} Define $k' = \min\{k, |Q|/2\}$, $a = \lceil |Q|/k' \rceil \geq 2$, and $\ell = \lceil \log_a n \rceil$. Compute labelling
    	\Statex\hspace{\algorithmicindent}scheme of \cite[Lemma 1]{shanmugam2015learning} on $Q$ with $(|Q|, k', a)$, and define $C_i = \{S_{x,y}\}_{x \in [\ell], y \in [a]}$,
    	\Statex\hspace{\algorithmicindent}where $S_{x,y} \subseteq Q$ is the subset of vertices whose $x^{th}$ letter in the label is $y$.
    	\State Update $i \gets i+1$, intervene on $C_{i}$ to obtain $\cE_{\cI_{i}}(G^*)$, and update $\cI_i \gets \cI_{i-1} \cup C_{i}$.
	\EndWhile
\end{algorithmic}
\end{algorithm}

\begin{lemma}
\label{lem:log-n-iterations}
\cref{alg:search-algo} terminates after at most $\cO(\log n)$ iterations.
\end{lemma}
\begin{proof}
Consider an arbitrary iteration $i$ and chain component $H$ with 1/2-clique separator $K_H$.

Edges within $K_H$ will be fully oriented by Step 6.
We now argue that any edge with exactly one endpoint in $K_H$ will be oriented by \cref{alg:search-algo}.
For $k'_H = 1$, the algorithm intervenes on all nodes in $K_H$ and thus such edges will be oriented trivially.
For $k'_H > 1$, the additional $\leq |K_H| / k'_H$ interventions \emph{after} the bounded size intervention strategy of \cite{shanmugam2015learning} ensures that \emph{every} edge with exactly one endpoint in $K_H$ will be separated.
Thus, after each iteration, the only remaining unoriented edges lie completely within the separated components that are of half the size.

Since the algorithm always recurse on graphs of size at least half the previous iteration, we see that $|H| \leq n/2^{i}$ for any $H \in CC(\cE_{\cI_{i}}(G^*))$.
Thus, all chain components will become singletons after $\cO(\log n)$ iterations and the algorithm terminates with a fully oriented graph.
\end{proof}

\searchatomic*
\begin{proof}
\cref{alg:search-algo} runs in polynomial time because 1/2-clique separators can be computed efficiently (see \cref{thm:chordal-separator}).

Fix an arbitrary iteration $i$ of \cref{alg:search-algo} and let $G_i$ be the partially oriented graph obtained after intervening on $\cI_i$.
By \cref{lem:strengthened-lb}, $\sum_{H \in CC(\cE_{\cI_i}(G^*))} \lfloor \frac{\omega(H)}{2} \rfloor \leq \nu_1(G^*)$.
By definition of $\omega$, we always have $|K_H| \leq \omega(H)$.
Thus, \cref{alg:search-algo} uses at most $2 \cdot \nu_1(G^*)$ interventions in each iteration.

By \cref{lem:log-n-iterations}, there are $\cO(\log n)$ iterations and so $\cO(\log (n) \cdot \nu_1(G^*))$ atomic interventions are used by \cref{alg:search-algo}.
\end{proof}

\begin{lemma}[Lemma 1 of \cite{shanmugam2015learning}]
Let $(n,k,a)$ be parameters where $k \leq n/2$.
There is a polynomial time labeling scheme that produces distinct $\ell$ length labels for all elements in $[n]$ using letters from the integer alphabet $\{0\} \cup [a]$ where $\ell = \lceil \log_a n \rceil$.
Further, in every digit (or position), any integer letter is used at most $\lceil n/a \rceil$ times.
This labelling scheme is a separating system: for any $i,j \in [n]$, there exists some digit $d \in [\ell]$ where the labels of $i$ and $j$ differ.
\end{lemma}

\searchbounded*
\begin{proof}
\cref{alg:search-algo} runs in polynomial time because 1/2-clique separators can be computed efficiently (see \cref{thm:chordal-separator}).
The label computation of \cite[Lemma 1]{shanmugam2015learning} also runs in polynomial time.

Fix an arbitrary iteration $i$ of \cref{alg:search-algo} and let $G_i$ be the partially oriented graph obtained after intervening on $\cI_i$.
Since $\cE(G_i) \subseteq \cE(G^*)$, we see that $\nu_1(G_i) \leq \nu_1(G^*)$.
\cref{alg:search-algo} intervenes on
\[
|C_{i}| \leq \left\lceil \frac{|Q|}{k'} \right\rceil \cdot \left\lceil \log_{\left\lceil \frac{|Q|}{k'} \right\rceil} |Q| \right\rceil
\]
sets of bounded size interventions, where $k' = \min\{k, |Q|/2\} > 1$.

Since $\nu_1(G^*) \geq \sum_{H \in CC(\cE(G^*))} \lfloor \omega(H)/2 \rfloor$, we know from \cref{lem:bounded-size-lb} that
\[
\nu_k(G^*)
\geq \left\lceil \frac{\nu_1(G^*)}{k} \right\rceil
\geq \left\lceil \frac{1}{k} \cdot \sum_{H \in CC(\cE(G^*))} \left\lfloor \frac{\omega(H)}{2} \right\rfloor \right\rceil \;.
\]
Since $Q$ is the union of clique separator nodes, we have that $|Q| \leq \sum_{H \in CC(\cE(G^*))} \omega(H)$ and so $\nu_k(G^*) \in \Omega(|Q|/k)$.
Note that $\nu_k(G^*) \geq 1$ always.

\textbf{Case 1: $k \leq |Q|/2$.}
Then, $k' = k$ and
\[
|C_{i}|
\leq \left\lceil \frac{|Q|}{k} \right\rceil \cdot \left\lceil \log_{\left\lceil \frac{|Q|}{k} \right\rceil} |Q| \right\rceil
\leq \left\lceil \frac{|Q|}{k} \right\rceil \cdot \left\lceil \frac{\log |Q|}{\log \frac{|Q|}{k}} \right\rceil
\leq \left\lceil \frac{|Q|}{k} \right\rceil \cdot \left\lceil \log (k) + 1 \right\rceil
\in \cO\left(\frac{|Q|}{k} \cdot \log (k) \right)
\]

\textbf{Case 2: $k \geq |Q|/2$.}
Then, $k' = |Q|/2$ and
\[
|C_{i}|
\leq 2 \cdot \left\lceil \log_{2} |Q| \right\rceil
\in \cO(\log (k))
\]

In either case, we see that $|C_i| \in \cO\left(\nu_k(G^*) \cdot \log k \right)$.
By \cref{lem:log-n-iterations}, there are $\cO(\log n)$ iterations and so $\cO(\log (n) \cdot \log (k) \cdot \nu_k(G^*))$ bounded size interventions are used by \cref{alg:search-algo}.
\end{proof}

\section{Example illustrating the power of \texorpdfstring{\cref{lem:strengthened-lb}}{Lemma 21}}
\label{sec:appendix-chain-triangles}

\begin{figure}[htbp]
\centering
\resizebox{\linewidth}{!}{%
\begin{tikzpicture}
%
%
\node[draw, circle, inner sep=2pt] at (0,0) (v1) {};
\node[draw, circle, inner sep=2pt, right=of v1] (v3) {};
\node[draw, circle, inner sep=2pt, above=of v3] (v2) {};

\node[draw, circle, inner sep=2pt, right=of v3] (v4) {};
\node[draw, circle, inner sep=2pt, right=of v4] (v6) {};
\node[draw, circle, inner sep=2pt, above=of v6] (v5) {};

\node[draw, circle, inner sep=2pt, right=50pt of v6] (v7) {};
\node[draw, circle, inner sep=2pt, right=of v7] (v9) {};
\node[draw, circle, inner sep=2pt, above=of v9] (v8) {};

\node[draw, circle, inner sep=2pt, right=of v9] (v10) {};
\node[draw, circle, inner sep=2pt, right=of v10] (v12) {};
\node[draw, circle, inner sep=2pt, above=of v12] (v11) {};

\node[draw, circle, inner sep=2pt, right=of v12] (v13) {};
\node[draw, circle, inner sep=2pt, right=of v13] (v15) {};
\node[draw, circle, inner sep=2pt, above=of v15] (v14) {};

\node[draw, circle, inner sep=2pt, right=50pt of v15] (v16) {};
\node[draw, circle, inner sep=2pt, right=of v16] (v18) {};
\node[draw, circle, inner sep=2pt, above=of v18] (v17) {};

\node[draw, circle, inner sep=2pt, right=of v18] (v19) {};
\node[draw, circle, inner sep=2pt, right=of v19] (v21) {};
\node[draw, circle, inner sep=2pt, above=of v21] (v20) {};

\node[inner sep=2pt] at ($(v6)!0.5!(v7)$) (dotsL) {$\ldots$};
\node[inner sep=2pt] at ($(v15)!0.5!(v16)$) (dotsR) {$\ldots$};

\draw[thick, -stealth, dashed] (v1) -- (v2);
\draw[thick, -stealth] (v1) -- (v3);
\draw[thick, -stealth, dashed] (v2) -- (v3);

\draw[thick, -stealth] (v3) -- (v4);

\draw[thick, -stealth] (v4) -- (v5);
\draw[thick, -stealth] (v4) -- (v6);
\draw[thick, -stealth, dashed] (v5) -- (v6);

\draw[thick] (v6) -- (dotsL);
\draw[thick, -stealth] (dotsL) -- (v7);

\draw[thick, -stealth] (v7) -- (v8);
\draw[thick, -stealth] (v7) -- (v9);
\draw[thick, -stealth, dashed] (v8) -- (v9);

\draw[thick, -stealth] (v9) -- (v10);

\draw[thick, -stealth] (v10) -- (v11);
\draw[thick, -stealth] (v10) -- (v12);
\draw[thick, -stealth, dashed] (v11) -- (v12);

\draw[thick, -stealth] (v12) -- (v13);

\draw[thick, -stealth] (v13) -- (v14);
\draw[thick, -stealth] (v13) -- (v15);
\draw[thick, -stealth, dashed] (v14) -- (v15);

\draw[thick] (v15) -- (dotsR);
\draw[thick, -stealth] (dotsR) -- (v16);

\draw[thick, -stealth] (v16) -- (v17);
\draw[thick, -stealth] (v16) -- (v18);
\draw[thick, -stealth, dashed] (v17) -- (v18);

\draw[thick, -stealth] (v18) -- (v19);

\draw[thick, -stealth] (v19) -- (v20);
\draw[thick, -stealth] (v19) -- (v21);
\draw[thick, -stealth, dashed] (v20) -- (v21);

%
%
\node[draw, circle, inner sep=2pt] at (0,-2) (ev1) {};
\node[draw, circle, inner sep=2pt, right=of ev1] (ev3) {};
\node[draw, circle, inner sep=2pt, above=of ev3] (ev2) {};

\node[draw, circle, inner sep=2pt, right=of ev3] (ev4) {};
\node[draw, circle, inner sep=2pt, right=of ev4] (ev6) {};
\node[draw, circle, inner sep=2pt, above=of ev6] (ev5) {};

\node[draw, circle, inner sep=2pt, right=50pt of ev6] (ev7) {};
\node[draw, circle, inner sep=2pt, right=of ev7] (ev9) {};
\node[draw, circle, inner sep=2pt, above=of ev9] (ev8) {};

\node[draw, circle, inner sep=2pt, right=of ev9] (ev10) {};
\node[draw, circle, inner sep=2pt, right=of ev10] (ev12) {};
\node[draw, circle, inner sep=2pt, above=of ev12] (ev11) {};

\node[draw, circle, inner sep=2pt, right=of ev12] (ev13) {};
\node[draw, circle, inner sep=2pt, right=of ev13] (ev15) {};
\node[draw, circle, inner sep=2pt, above=of ev15] (ev14) {};

\node[draw, circle, inner sep=2pt, right=50pt of ev15] (ev16) {};
\node[draw, circle, inner sep=2pt, right=of ev16] (ev18) {};
\node[draw, circle, inner sep=2pt, above=of ev18] (ev17) {};

\node[draw, circle, inner sep=2pt, right=of ev18] (ev19) {};
\node[draw, circle, inner sep=2pt, right=of ev19] (ev21) {};
\node[draw, circle, inner sep=2pt, above=of ev21] (ev20) {};

\node[inner sep=2pt] at ($(ev6)!0.5!(ev7)$) (edotsL) {$\ldots$};
\node[inner sep=2pt] at ($(ev15)!0.5!(ev16)$) (edotsR) {$\ldots$};

\draw[thick] (ev1) -- (ev2);
\draw[thick] (ev1) -- (ev3);
\draw[thick] (ev2) -- (ev3);

\draw[thick] (ev3) -- (ev4);

\draw[thick] (ev4) -- (ev5);
\draw[thick] (ev4) -- (ev6);
\draw[thick] (ev5) -- (ev6);

\draw[thick] (ev6) -- (edotsL);
\draw[thick] (edotsL) -- (ev7);

\draw[thick] (ev7) -- (ev8);
\draw[thick] (ev7) -- (ev9);
\draw[thick] (ev8) -- (ev9);

\draw[thick] (ev9) -- (ev10);

\draw[thick] (ev10) -- (ev11);
\draw[thick] (ev10) -- (ev12);
\draw[thick] (ev11) -- (ev12);

\draw[thick] (ev12) -- (ev13);

\draw[thick] (ev13) -- (ev14);
\draw[thick] (ev13) -- (ev15);
\draw[thick] (ev14) -- (ev15);

\draw[thick] (ev15) -- (edotsR);
\draw[thick] (edotsR) -- (ev16);

\draw[thick] (ev16) -- (ev17);
\draw[thick] (ev16) -- (ev18);
\draw[thick] (ev17) -- (ev18);

\draw[thick] (ev18) -- (ev19);

\draw[thick] (ev19) -- (ev20);
\draw[thick] (ev19) -- (ev21);
\draw[thick] (ev20) -- (ev21);

%
%
\node[draw, circle, inner sep=2pt] at (0,-4) (iv1) {};
\node[draw, circle, inner sep=2pt, right=of iv1] (iv3) {};
\node[draw, circle, inner sep=2pt, above=of iv3] (iv2) {};

\node[draw, circle, inner sep=2pt, right=of iv3] (iv4) {};
\node[draw, circle, inner sep=2pt, right=of iv4] (iv6) {};
\node[draw, circle, inner sep=2pt, above=of iv6] (iv5) {};

\node[draw, circle, inner sep=2pt, right=50pt of iv6] (iv7) {};
\node[draw, circle, inner sep=2pt, right=of iv7] (iv9) {};
\node[draw, circle, inner sep=2pt, above=of iv9] (iv8) {};

\node[draw, circle, inner sep=2pt, right=of iv9] (iv10) {};
\node[draw, circle, inner sep=2pt, right=of iv10] (iv12) {};
\node[draw, circle, inner sep=2pt, above=of iv12] (iv11) {};

\node[draw, circle, inner sep=2pt, right=of iv12] (iv13) {};
\node[draw, circle, inner sep=2pt, right=of iv13] (iv15) {};
\node[draw, circle, inner sep=2pt, above=of iv15] (iv14) {};

\node[draw, circle, inner sep=2pt, right=50pt of iv15] (iv16) {};
\node[draw, circle, inner sep=2pt, right=of iv16] (iv18) {};
\node[draw, circle, inner sep=2pt, above=of iv18] (iv17) {};

\node[draw, circle, inner sep=2pt, right=of iv18] (iv19) {};
\node[draw, circle, inner sep=2pt, right=of iv19] (iv21) {};
\node[draw, circle, inner sep=2pt, above=of iv21] (iv20) {};

\node[inner sep=2pt] at ($(iv6)!0.5!(iv7)$) (idotsL) {$\ldots$};
\node[inner sep=2pt] at ($(iv15)!0.5!(iv16)$) (idotsR) {$\ldots$};

\draw[thick] (iv1) -- (iv2);
\draw[thick] (iv1) -- (iv3);
\draw[thick] (iv2) -- (iv3);

\draw[thick] (iv3) -- (iv4);

\draw[thick] (iv4) -- (iv5);
\draw[thick] (iv4) -- (iv6);
\draw[thick] (iv5) -- (iv6);

\draw[thick] (iv6) -- (idotsL);
\draw[thick] (idotsL) -- (iv7);

\draw[thick] (iv7) -- (iv8);
\draw[thick] (iv7) -- (iv9);
\draw[thick] (iv8) -- (iv9);

\draw[thick, -stealth] (iv9) -- (iv10);

\draw[thick, -stealth] (iv10) -- (iv11);
\draw[thick, -stealth] (iv10) -- (iv12);
\draw[thick, -stealth] (iv11) -- (iv12);

\draw[thick, -stealth] (iv12) -- (iv13);

\draw[thick, -stealth] (iv13) -- (iv14);
\draw[thick, -stealth] (iv13) -- (iv15);
\draw[thick] (iv14) -- (iv15);

\draw[thick, -stealth] (iv15) -- (idotsR);
\draw[thick, -stealth] (idotsR) -- (iv16);

\draw[thick, -stealth] (iv16) -- (iv17);
\draw[thick, -stealth] (iv16) -- (iv18);
\draw[thick] (iv17) -- (iv18);

\draw[thick, -stealth] (iv18) -- (iv19);

\draw[thick, -stealth] (iv19) -- (iv20);
\draw[thick, -stealth] (iv19) -- (iv21);
\draw[thick] (iv20) -- (iv21);

%
%
\node[draw, circle, inner sep=2pt] at (0,-6) (iev1) {};
\node[draw, circle, inner sep=2pt, right=of iev1] (iev3) {};
\node[draw, circle, inner sep=2pt, above=of iev3] (iev2) {};

\node[draw, circle, inner sep=2pt, right=of iev3] (iev4) {};
\node[draw, circle, inner sep=2pt, right=of iev4] (iev6) {};
\node[draw, circle, inner sep=2pt, above=of iev6] (iev5) {};

\node[draw, circle, inner sep=2pt, right=50pt of iev6] (iev7) {};
\node[draw, circle, inner sep=2pt, right=of iev7] (iev9) {};
\node[draw, circle, inner sep=2pt, above=of iev9] (iev8) {};

\node[draw, circle, inner sep=2pt, right=of iev9] (iev10) {};
\node[draw, circle, inner sep=2pt, right=of iev10] (iev12) {};
\node[draw, circle, inner sep=2pt, above=of iev12] (iev11) {};

\node[draw, circle, inner sep=2pt, right=of iev12] (iev13) {};
\node[draw, circle, inner sep=2pt, right=of iev13] (iev15) {};
\node[draw, circle, inner sep=2pt, above=of iev15] (iev14) {};

\node[draw, circle, inner sep=2pt, right=50pt of iev15] (iev16) {};
\node[draw, circle, inner sep=2pt, right=of iev16] (iev18) {};
\node[draw, circle, inner sep=2pt, above=of iev18] (iev17) {};

\node[draw, circle, inner sep=2pt, right=of iev18] (iev19) {};
\node[draw, circle, inner sep=2pt, right=of iev19] (iev21) {};
\node[draw, circle, inner sep=2pt, above=of iev21] (iev20) {};

\node[inner sep=2pt] at ($(iev6)!0.5!(iev7)$) (iedotsL) {$\ldots$};
\node[inner sep=2pt] at ($(iev15)!0.5!(iev16)$) (iedotsR) {$\ldots$};

\draw[thick] (iev1) -- (iev2);
\draw[thick] (iev1) -- (iev3);
\draw[thick] (iev2) -- (iev3);

\draw[thick] (iev3) -- (iev4);

\draw[thick] (iev4) -- (iev5);
\draw[thick] (iev4) -- (iev6);
\draw[thick] (iev5) -- (iev6);

\draw[thick] (iev6) -- (iedotsL);
\draw[thick] (iedotsL) -- (iev7);

\draw[thick] (iev7) -- (iev8);
\draw[thick] (iev7) -- (iev9);
\draw[thick] (iev8) -- (iev9);

\draw[thick] (iev14) -- (iev15);

\draw[thick] (iev17) -- (iev18);

\draw[thick] (iev20) -- (iev21);

%
%
\node[] at ($(v2)!0.5!(v3) + (-2,0)$) {$G^*$};
\node[] at ($(ev2)!0.5!(ev3) + (-2,0)$) {$\cE(G^*)$};
\node[] at ($(iv2)!0.5!(iv3) + (-2,0)$) {$\cE_{\cI}(G^*)$};
\node[] at ($(iev2)!0.5!(iev3) + (-2,0)$) {$CC(\cE_{\cI}(G^*))$};
\node[fit=(ev10)(ev11)(ev12), rounded corners, draw, inner sep=10pt] {};
\end{tikzpicture}
}
\caption{
A DAG $G^*$ where minimum vertex cover of the unoriented covered edges (dashed arcs) is much larger than the size of the maximal clique (triangle): $\nu_1(G^*) \approx n$ while the lower bound of \cite{squires2020active} on $\cE(G^*)$ is a constant.
Let $\cI$ be an atomic intervention set on the middle triangle, i.e.\ the three vertices boxed up in $\cE(G^*)$.
The partially directed graph $\cE_{\cI}(G^*)$ shows the learnt arc directions after intervening on $\cI$ and applying Meek rules.
Applying the lower bound of \cite{squires2020active} on $CC(\cE_{\cI}(G^*))$ now gives a much stronger lower bound of $\approx n$ due to the single edge components.
}
\label{fig:chain-triangles}
\end{figure}
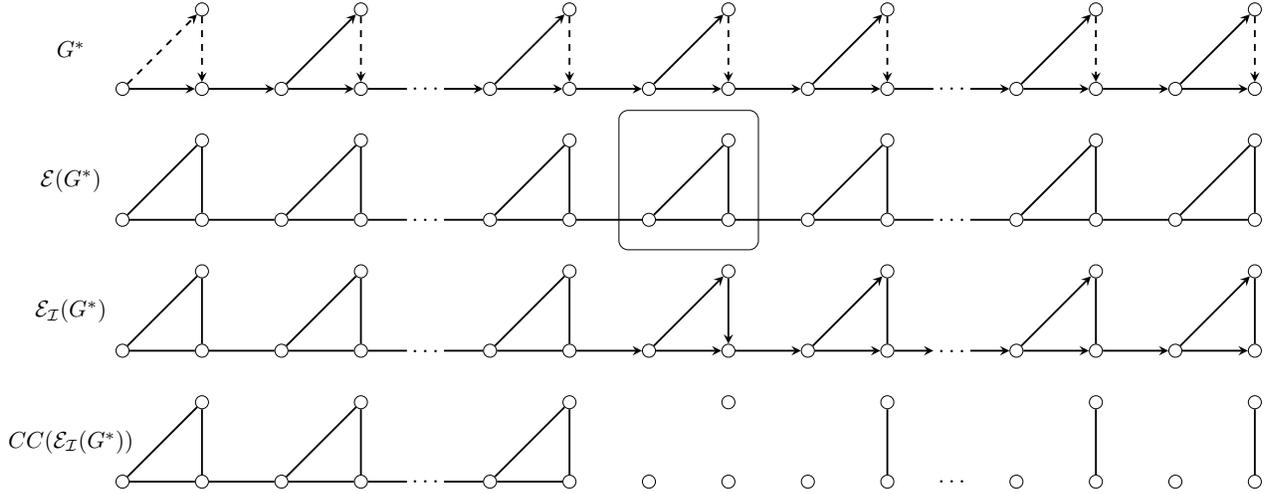
\section{Experiments and implementation}
\label{sec:appendix-experiments}

Our code and entire experimental setup is available at \url{https://github.com/cxjdavin/verification-and-search-algorithms-for-causal-DAGs}.

\subsection{Implementation details}

\paragraph{Verification}
We implemented our verification algorithm and tested its correctness on some well-known graphs such as cliques and trees for which we know the exact verification number.

\paragraph{Search}
We use \texttt{FAST CHORDAL SEPARATOR} algorithm in \cite{gilbert1984separatorchordal} to compute a chordal graph separator. This algorithm first computes a perfect elimination ordering of a given chordal graph and we use Eppstein's \texttt{LexBFS} implementation (\url{https://www.ics.uci.edu/~eppstein/PADS/LexBFS.py}) to compute such an ordering.

\subsection{Experiments}

We base our evaluation on the experimental framework of \cite{squires2020active} (\url{https://github.com/csquires/dct-policy}) which empirically compares atomic intervention policies.
The experiments are conducted on an Ubuntu server with two AMD EPYC 7532 CPU and 256GB DDR4 RAM.

\subsubsection{Synthetic graph classes}

The synthetic graphs are random connected DAGs whose essential graph is a single chain component (i.e.\ moral DAGs in \cite{squires2020active}'s terminology).
Below, we reproduce the synthetic graph generation procedure from \cite[Section 5]{squires2020active}.

\begin{enumerate}
    \item Erd\H{o}s-R\'{e}nyi styled graphs\\
    These graphs are parameterized by 2 parameters: $n$ and density $\rho$.
    Generate a random ordering $\sigma$ over $n$ vertices.
    Then, set the in-degree of the $n^{th}$ vertex (i.e.\ last vertex in the ordering) in the order to be $X_n = \max\{1, \texttt{Binomial}(n-1, \rho)\}$, and sample $X_n$ parents uniformly form the nodes earlier in the ordering.
    Finally, chordalize the graph by running the elimination algorithm of \cite{koller2009probabilistic} with elimination ordering equal to the reverse of $\sigma$.
    \item Tree-like graphs\\
    These graphs are parameterized by 4 parameters: $n$, degree $d$, $e_{\min}$, and $e_{\max}$. First, generate a complete directed $d$-ary tree on $n$ nodes.
    Then, add $\texttt{Uniform}(e_{\min}, e_{\max})$ edges to the tree.
    Finally, compute a topological order of the graph by DFS and triangulate the graph using that order.
\end{enumerate}

\subsubsection{Algorithms benchmarked}

The following algorithms perform \emph{atomic interventions}.
Our algorithm \texttt{separator} perform atomic interventions when given $k=1$ and \emph{bounded size interventions} when given $k > 1$.
\begin{description}
    \item[\texttt{random}:] A baseline algorithm that repeatedly picks a random non-dominated node (a node that is incident to some unoriented edge) from the interventional essential graph
    \item[\texttt{dct}:] \texttt{DCT Policy} of \cite{squires2020active}
    \item[\texttt{coloring}:] \texttt{Coloring} of \cite{shanmugam2015learning}
    \item[\texttt{opt\_single}:] \texttt{OptSingle} of \cite{hauser2014two}
    \item[\texttt{greedy\_minmax}:] \texttt{MinmaxMEC} of \cite{he2008active}
    \item[\texttt{greedy\_entropy}:] \texttt{MinmaxEntropy} of \cite{he2008active}
    \item[\texttt{separator}:] Our \cref{alg:search-algo}. It takes in a parameter $k$ to serve as an upper bound on the number of vertices to use in an intervention.
\end{description}

\subsubsection{Metrics measured}

Each experiment produces 4 plots measuring ``average competitive ratio'', ``maximum competitive ratio'', ``intervention count'', and ``time taken''.
For any fixed setting, 100 synthetic DAGs are generated as $G^*$ for testing, so we include error bars for ``average competitive ratio'', ``average intervention count'', and ``time taken'' in the plots.
For all metrics, ``lower is better''.

The competitive ratio for an input DAG $G^*$ is measured in terms of \emph{total atomic interventions used to orient the essential graph of $G^*$ to become $G^*$}, divided by \emph{minimum number of interventions needed to orient $G^*$} (i.e. the verification number $\nu_1(G^*)$ of $G^*$).
For non-atomic interventions, we know (\cref{lem:bounded-size-lb}) that $\nu_k(G^*) \geq \lceil \nu_1(G^*)/k \rceil$, so we use $\lceil \nu_1(G^*)/k \rceil$ as the denominator of the competitive ratio computation.
While the competitive ratio increases as $k$ increases (\cref{thm:search-bounded}), the number of interventions used decreases as $k$ increases.

Time is measured as the total amount of time taken to finish computing the nodes to intervene and performing the interventions.
Note that our algorithm can beat \texttt{random} in terms of runtime in some cases because \texttt{random} uses significantly more interventions and hence more overall computation.

\subsubsection{Experimental results}

Qualitatively, our \cref{alg:search-algo} with $k=1$ has a similar competitive ratio to the current best-known atomic intervention policies in the literature (\texttt{DCT} and \texttt{Coloring}) while running significantly faster for some graphs (roughly $\sim$10x faster on tree-like graphs).

\paragraph{Experiment 1}
Graph class 1 with $n \in \{10, 15, 20, 25\}$ and density $\rho = 0.1$.
This is the same setup as \cite{squires2020active}.
Additionally, we run \cref{alg:search-algo} with $k = 1$.
See \cref{fig:exp1}.

\begin{figure}[htbp]
\centering
\begin{subfigure}[b]{0.4\textwidth}
    \centering
    \includegraphics[width=\textwidth]{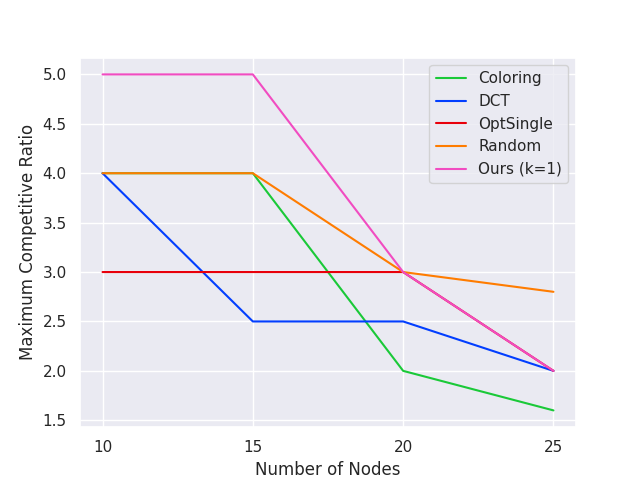}
    \caption{Max competitive ratio}
    \label{fig:maxcompratio1}
\end{subfigure}
\begin{subfigure}[b]{0.4\textwidth}
    \centering
    \includegraphics[width=\textwidth]{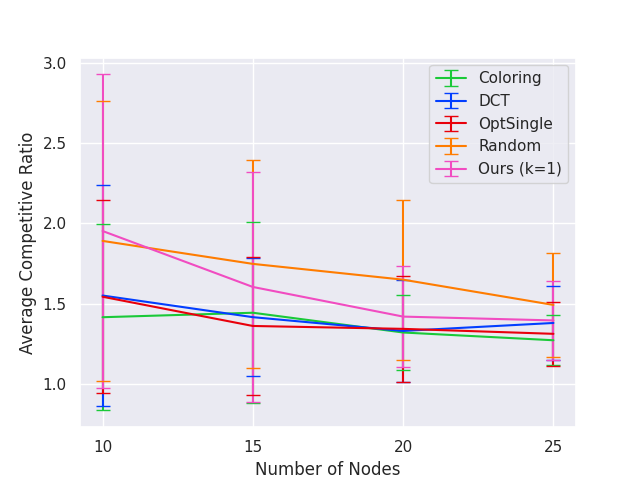}
    \caption{Average competitive ratio}
    \label{fig:avgcompratio1}
\end{subfigure}
\begin{subfigure}[b]{0.4\textwidth}
    \centering
    \includegraphics[width=\textwidth]{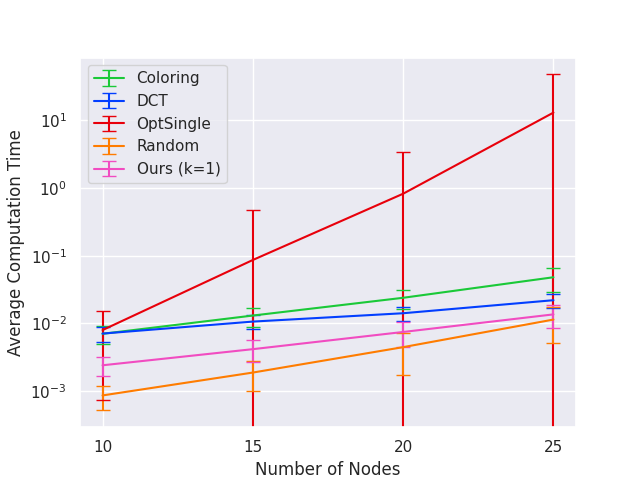}
    \caption{Time}
    \label{fig:time1}
\end{subfigure}
\begin{subfigure}[b]{0.4\textwidth}
    \centering
    \includegraphics[width=\textwidth]{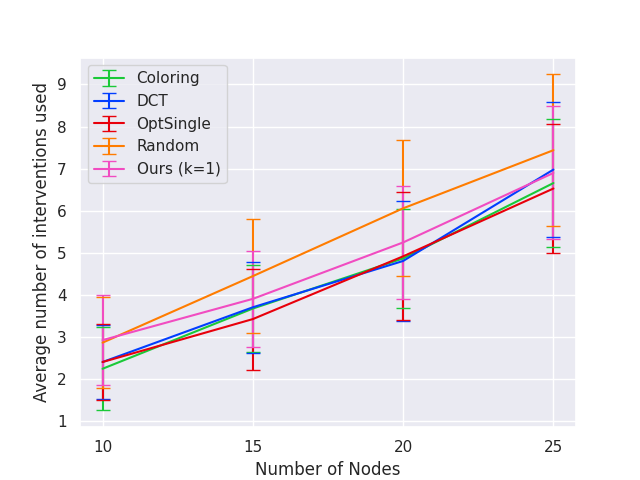}
    \caption{Average interventions count}
    \label{fig:interventions1}
\end{subfigure}
\caption{Plots for experiment 1}
\label{fig:exp1}
\end{figure}

\paragraph{Experiment 2}
Graph class 1 with $n \in \{8, 10, 12, 14\}$ and density $\rho = 0.1$.
This is the same setup as \cite{squires2020active}.
Additionally, we run \cref{alg:search-algo} with $k = 1$.
Note that this is the same graph class as experiment 1, but on smaller graphs because some slower algorithms are being run.
See \cref{fig:exp2}.

\begin{figure}[htbp]
\centering
\begin{subfigure}[b]{0.4\textwidth}
    \centering
    \includegraphics[width=\textwidth]{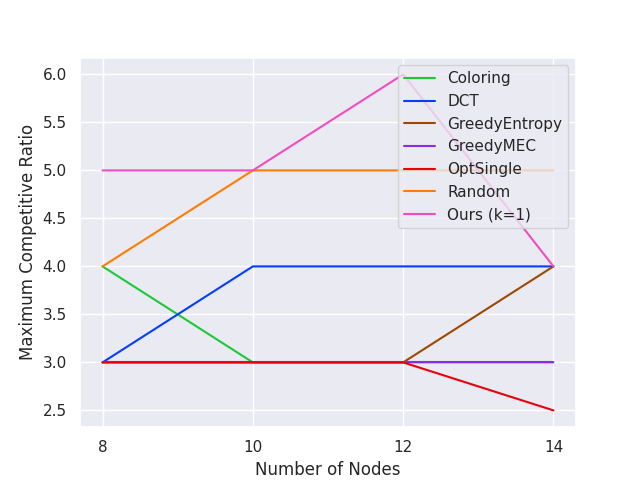}
    \caption{Max competitive ratio}
    \label{fig:maxcompratio2}
\end{subfigure}
\begin{subfigure}[b]{0.4\textwidth}
    \centering
    \includegraphics[width=\textwidth]{plots/exp2_avgcompratio.png}
    \caption{Average competitive ratio}
    \label{fig:avgcompratio2}
\end{subfigure}
\begin{subfigure}[b]{0.4\textwidth}
    \centering
    \includegraphics[width=\textwidth]{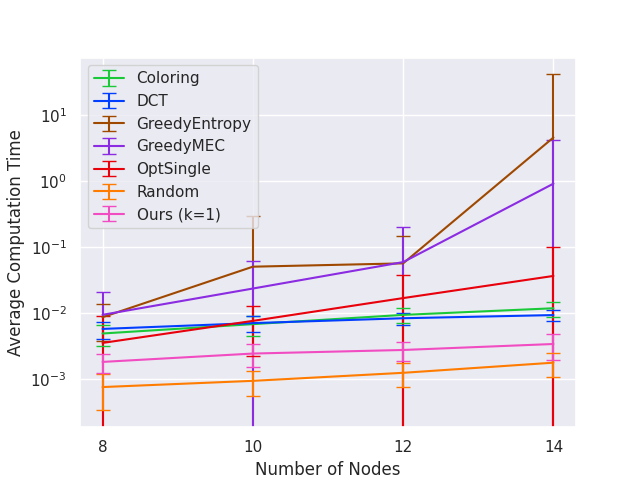}
    \caption{Time}
    \label{fig:time2}
\end{subfigure}
\begin{subfigure}[b]{0.4\textwidth}
    \centering
    \includegraphics[width=\textwidth]{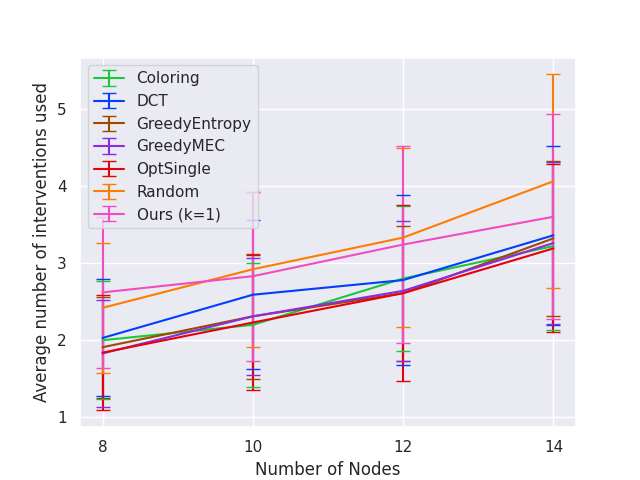}
    \caption{Average interventions count}
    \label{fig:interventions2}
\end{subfigure}
\caption{Plots for experiment 2}
\label{fig:exp2}
\end{figure}

\paragraph{Experiment 3}
Graph class 2 with $n \in \{100, 200, 300, 400, 500\}$ and $(\text{degree}, e_{\min}, e_{\max}) = (4, 2, 5)$.
This is the same setup as \cite{squires2020active}.
Additionally, we run \cref{alg:search-algo} with $k = 1$.
See \cref{fig:exp3}.

\begin{figure}[htbp]
\centering
\begin{subfigure}[b]{0.4\textwidth}
    \centering
    \includegraphics[width=\textwidth]{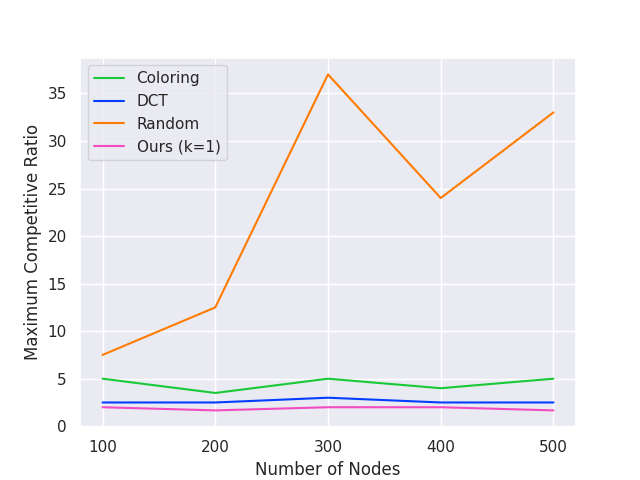}
    \caption{Max competitive ratio}
    \label{fig:maxcompratio3}
\end{subfigure}
\begin{subfigure}[b]{0.4\textwidth}
    \centering
    \includegraphics[width=\textwidth]{plots/exp3_avgcompratio.png}
    \caption{Average competitive ratio}
    \label{fig:avgcompratio3}
\end{subfigure}
\begin{subfigure}[b]{0.4\textwidth}
    \centering
    \includegraphics[width=\textwidth]{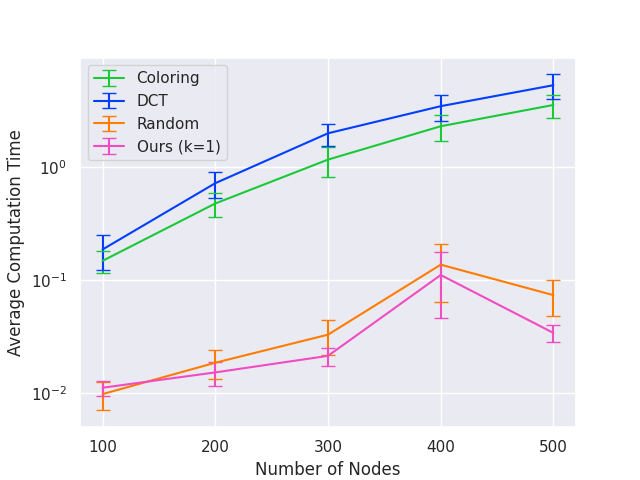}
    \caption{Time}
    \label{fig:time3}
\end{subfigure}
\begin{subfigure}[b]{0.4\textwidth}
    \centering
    \includegraphics[width=\textwidth]{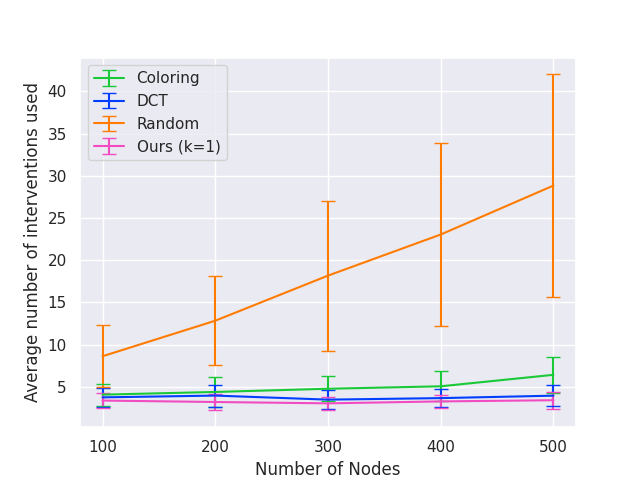}
    \caption{Average interventions count}
    \label{fig:interventions3}
\end{subfigure}
\caption{Plots for experiment 3}
\label{fig:exp3}
\end{figure}

\paragraph{Experiment 4}
Graph class 1 with $n \in \{10, 15, 20, 25\}$ and density $\rho = 0.1$.
We run \cref{alg:search-algo} with $k \in$ {1,2,3,5} on the same graph class as experiment 1, but on larger graphs.
See \cref{fig:exp4}.

\begin{figure}[htbp]
\centering
\begin{subfigure}[b]{0.4\textwidth}
    \centering
    \includegraphics[width=\textwidth]{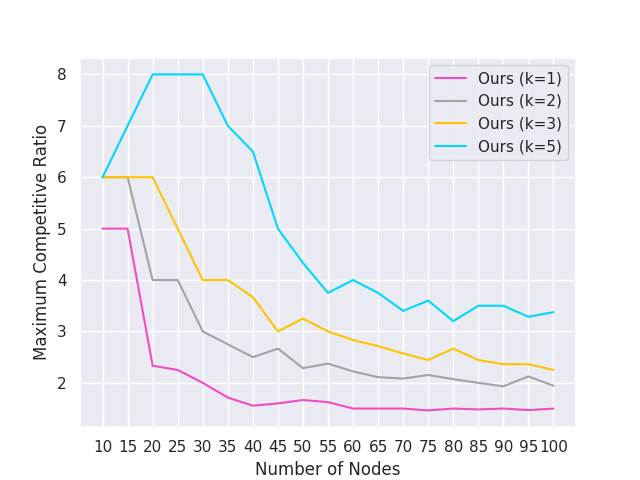}
    \caption{Max competitive ratio}
    \label{fig:maxcompratio4}
\end{subfigure}
\begin{subfigure}[b]{0.4\textwidth}
    \centering
    \includegraphics[width=\textwidth]{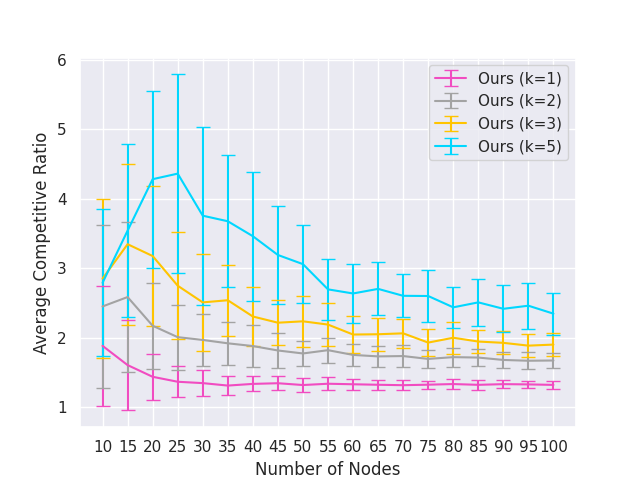}
    \caption{Average competitive ratio}
    \label{fig:avgcompratio4}
\end{subfigure}
\begin{subfigure}[b]{0.4\textwidth}
    \centering
    \includegraphics[width=\textwidth]{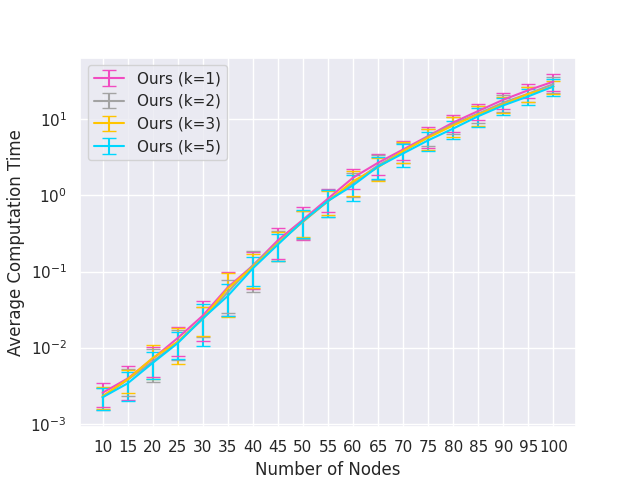}
    \caption{Time}
    \label{fig:time4}
\end{subfigure}
\begin{subfigure}[b]{0.4\textwidth}
    \centering
    \includegraphics[width=\textwidth]{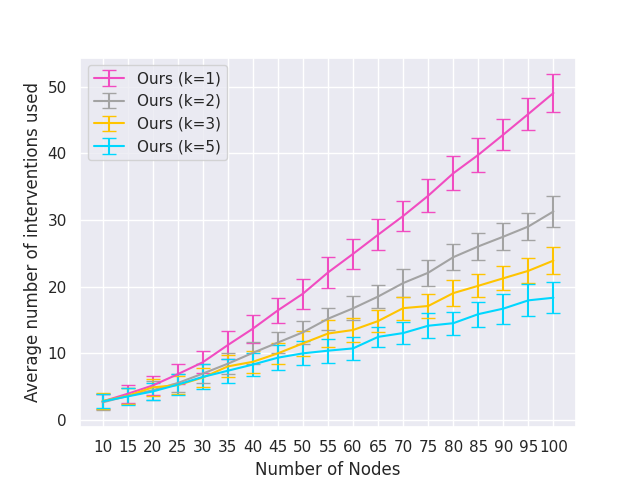}
    \caption{Average interventions count}
    \label{fig:interventions4}
\end{subfigure}
\caption{Plots for experiment 4}
\label{fig:exp4}
\end{figure}

\paragraph{Experiment 5}
Graph class 2 with $n \in \{100, 200, 300, 400, 500\}$ and $(\text{degree}, e_{\min}, e_{\max}) = (40, 20, 50)$.
We run \cref{alg:search-algo} with $k \in \{1,2,3,5\}$ on the same graph class as experiment 3, but on denser graphs.
See \cref{fig:exp5}.

\begin{figure}[htbp]
\centering
\begin{subfigure}[b]{0.4\textwidth}
    \centering
    \includegraphics[width=\textwidth]{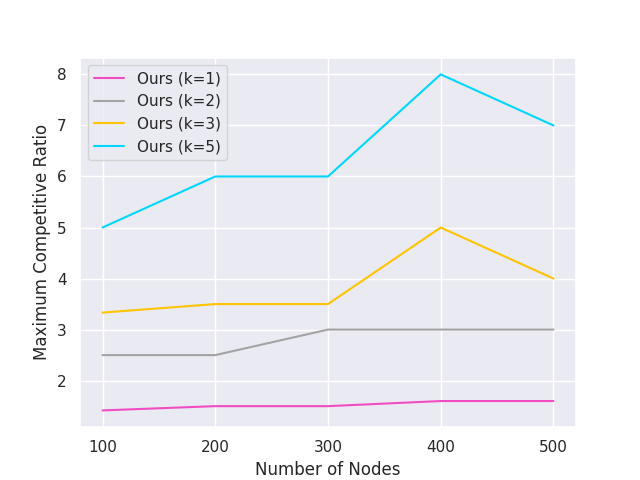}
    \caption{Max competitive ratio}
    \label{fig:maxcompratio5}
\end{subfigure}
\begin{subfigure}[b]{0.4\textwidth}
    \centering
    \includegraphics[width=\textwidth]{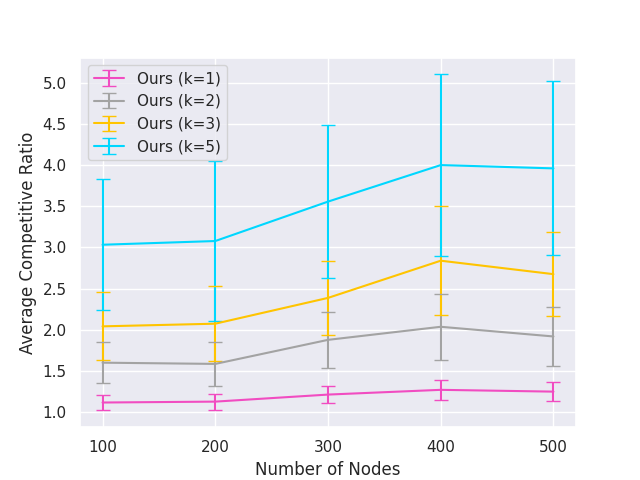}
    \caption{Average competitive ratio}
    \label{fig:avgcompratio5}
\end{subfigure}
\begin{subfigure}[b]{0.4\textwidth}
    \centering
    \includegraphics[width=\textwidth]{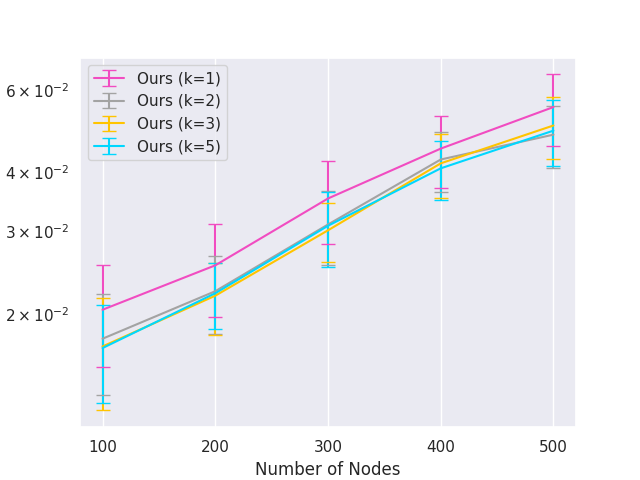}
    \caption{Time}
    \label{fig:time5}
\end{subfigure}
\begin{subfigure}[b]{0.4\textwidth}
    \centering
    \includegraphics[width=\textwidth]{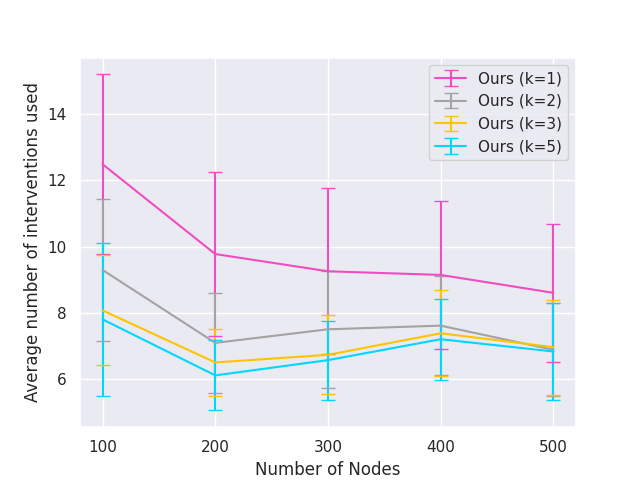}
    \caption{Average interventions count}
    \label{fig:interventions5}
\end{subfigure}
\caption{Plots for experiment 5}
\label{fig:exp5}
\end{figure}

\end{document}